\documentclass[12pt]{article}
\usepackage[utf8]{inputenc}
\usepackage{fullpage}
\usepackage{times}
\usepackage[normalem]{ulem}
\usepackage{fancyhdr,graphicx,amsmath,amssymb, mathtools, scrextend, titlesec, enumitem}
\usepackage{bm}
\usepackage{tcolorbox}
\usepackage{amsthm}
\usepackage{mathrsfs}  

\usepackage{amsfonts}

\usepackage[hypertexnames=false]{hyperref}
\hypersetup{
    colorlinks = true,
    linkbordercolor = {white},
    linkcolor = {blue!50!black},
citecolor     = {blue!50!black},
}
\usepackage[nameinlink,capitalize]{cleveref}
\crefname{algocf}{Algorithm}{Algorithms}

\numberwithin{equation}{section}
\newtheorem{theorem}{Theorem}[section]
\newtheorem{definition}[theorem]{Definition}
\newtheorem{problem}[theorem]{Problem}

\newtheorem{corollary}[theorem]{Corollary}
\newtheorem{remark}[theorem]{Remark}
\newtheorem{lemma}[theorem]{Lemma}

\newtheorem{assumption}[theorem]{Assumption}

\usepackage{caption}
\usepackage{subcaption}

\usepackage{algpseudocode}
\usepackage{tikz}

\usepackage{easy-todo} 

\usepackage{chngcntr}
\counterwithin{figure}{section}

\usepackage{ulem}

\usepackage{yhmath}

\makeatletter
\@namedef{subjclassname@2020}{%
  \textup{2020} Mathematics Subject Classification}
\makeatother

\usepackage{enumitem}
\usepackage{arxiv}
\usepackage[utf8]{inputenc}
\usepackage{ stmaryrd }

\usepackage[ruled]{algorithm2e}
\usepackage{bm}

\usepackage{svg}

\newcommand*{\obs}{\ensuremath{\varphi}}
\newcommand*{\bfobs}{\ensuremath{\boldsymbol{\varphi}}}

\newcommand{\scpr}[1]{\left\langle #1 \right\rangle}

\newcommand{\spa}{\text{span}}

\renewcommand*{\phi}{\varphi}
\renewcommand*{\rho}{\varrho}

\renewcommand*{\subset}{\subseteq}

\newcommand*{\dd}{\ensuremath{\mathrm{d}}}
\newcommand*{\dx}[1]{\ensuremath{\,\dd{#1}}}

\DeclareMathOperator{\Id}{Id}

\DeclareMathOperator{\supp}{supp}

\newcommand*{\abs}[2][{}]{\ensuremath{#1\left\lvert#2#1\right\rvert}}
\newcommand*{\norm}[2][{}]{\ensuremath{#1\left\Vert#2#1\right\Vert}}

\newcommand*{\chl}{\ensuremath{C_{V_h}}}
\newcommand*{\lb}{\ensuremath{\mathfrak{c}}}
\newcommand*{\ub}{\ensuremath{\mathfrak{C}}}

\newcommand*{\bfkappa}{{\ensuremath{\boldsymbol{\kappa}}}}

\newcommand*{\bfu}{\ensuremath{\mathbf{u}}}
\newcommand*{\bfx}{\ensuremath{\mathbf{x}}}

\newcommand*{\bfy}{\ensuremath{\mathbf{y}}}
\newcommand*{\bfv}{\ensuremath{\mathbf{v}}}
\newcommand*{\bfw}{\ensuremath{\mathbf{w}}}

\newcommand*{\bff}{\ensuremath{\mathbf{f}}}

\newcommand*{\bfr}{\ensuremath{\mathbf{r}}}
\newcommand*{\bfe}{\ensuremath{\mathbf{e}}}

\newcommand*{\bfups}{{\ensuremath{\boldsymbol{\Upsilon}}}}

\newcommand{\vcmrm}{\ensuremath{\mathrm{VCMR}_{k,k_0,\ell}^m}}

\newcommand*{\fesp}{\ensuremath{V_h}}

\newcommand*{\T}{\ensuremath{\mathcal{T}}}
\newcommand*{\N}{\ensuremath{\mathcal{N}}}
\newcommand*{\Rl}{\ensuremath{R_\ell^{\max}}}

\usepackage{graphicx} 
\usepackage{nicefrac}
\usepackage{array}
\usepackage{booktabs}

\title{Multi-level Neural Networks for high-dimensional parametric obstacle problems}
\renewcommand{\shorttitle}{Multi-level CNNs for parametric obstacle problems}
\author{Martin Eigel$^1$, Cosmas Hei\ss$^2$, Janina E. Sch\"utte$^1$ \\
  $^1$ Weierstraß Insitute of Applied Analyis and Stochastics\\
  $^2$ École Polytechnique Fédérale de Lausanne\\}

\begin{document}

\maketitle

\begin{abstract}
    A new method to solve computationally challenging (random) parametric obstacle problems is developed and analyzed, where the parameters can influence the related partial differential equation (PDE) and determine the position and surface structure of the obstacle.
    As governing equation, a stationary elliptic diffusion problem is assumed.
    The high-dimensional solution of the obstacle problem is approximated by a specifically constructed convolutional neural network (CNN).
    This novel algorithm is inspired by 
    a finite element constrained multigrid algorithm 
    to represent the parameter to solution map.
    This has two benefits: First, it allows for efficient practical computations since multi-level data is used as an explicit output of the NN thanks to an appropriate data preprocessing.
    This improves the efficacy of the training process and subsequently leads to small errors in the natural energy norm.
    Second, the comparison of the CNN to a multigrid algorithm provides means to carry out a complete a priori convergence and complexity analysis of the proposed NN architecture.
    Numerical experiments illustrate a state-of-the-art performance for this challenging problem. 
\end{abstract}

\section{Introduction}
Free boundary problems arise in different research and engineering areas.
They constitute solutions of a PDE with a priori unknown boundaries.
Well-known examples include classes of obstacle problems and variational inequalities. 
The solution of a classical obstacle problem describes the position of an elastic membrane as a function $u$, which is fixed on the boundary of the domain $D$, always lies above some known obstacle $\obs$ and is under the influence of some forcing $f$, see~\cite{rosoton2017obstacleproblemsfreeboundaries}.
The interface where the membrane touches the obstacle is a priori unknown.
On the part of the domain where the membrane hangs freely, the position function fulfills a stationary diffusion equation.
Fixing $u$ on the boundary, the problem considered here has the form
\begin{equation}\label{equation: VI}
    \begin{cases}
        u(x, \bfy) \geq \obs(x, \bfy) &\text{for all } x\in D\\
        -\nabla \cdot (\kappa(x,\bfy) \nabla u(x,\bfy)) = f(x,\bfy) &\text{for all } x \text{ such that } u(x,\bfy) > \obs(x,\bfy)\\
        -\nabla \cdot (\kappa(x,\bfy) \nabla u(x,\bfy)) \geq f(x,\bfy) &\text{for all } x\in D
    \end{cases},
\end{equation}
depending on some coefficient $\kappa$ and some countably infinite dimensional parameter vector $\bfy\in\Gamma\subset\mathbb{R}^\mathbb{N}$. The dependence on $\kappa$ has for instance also been considered in~\cite{ForsterKornhuber,kornhuber2014}. 
Obstacle problems are found in a variety of applications, namely the Stefan problem describing the process of ice melting in water, can be rewritten in form of the obstacle problem~\cite{duvaut}.
Furthermore, the obstacle problem finds applications in financial mathematics~\cite{BLANCHET20061362, BLANCHET2006371, laurence} and the minimizer of interaction energies can be written in terms of the solution of an obstacle problem~\cite{carillo}, which are encountered e.g. in in physics in particle behavior \cite{Holm2005AggregationOF, HOLM2006183} or in biology in collective behavior of animals~\cite{Bernoff, D'Orsogna, Fetecau}.
For more information on practical uses of the obstacle problem, we refer to~\cite{friedman, borjan,rosoton2017obstacleproblemsfreeboundaries} and references therein.

With the gaining popularity of scientific machine learning (SciML), neural networks have been applied to solve obstacle problems in different ways, where the equations can be incorporated into a loss function. The obstacle condition can either be encoded in the neural network architecture, as e.g. in the first approach in~\cite{Zhao_2022} or in~\cite{bahja2024physicsinformedneuralnetworkframework}, or it can be enforced by penelization, as e.g. in the second approach in~\cite{Zhao_2022} or in~\cite{CHENG2023103864}.
In \cite{alphonse2024neuralnetworkapproachlearning}, the variational formulation of the obstacle problem is rewritten in a min-max formulation that is then used for training. 
 
In the work presented here the coefficient and the obstacle depend on some high-dimensional stochastic parameter vector.
In this setting the obstacle problem has to be solved for a large number of realizations of the vector. Classical mathematical solvers, as for example presented in \cite{solverforobstacle} apply to a large class of coefficients. We exploit and inherent distribution of the parameters to develop more efficient tuned surrogate models  
mapping realizations of $\kappa$ and $\obs$ to solutions of the obstacle problem.

A surrogate model mapping $\obs$ to the solution has been derived in~\cite{schwab}.
The proximal neural network architecture with activation functions enforcing the obstacle condition was analyzed in the more general setting of variational inequalities. The convergence achieved in~\cite[Theorem 4.2]{schwab} based on a fixed $\kappa$ and $\obs$ is comparable to the convergence achieved in the present work for variable $\kappa$ and $\obs$ as the analysis is based on an iterative scheme to approximate the solution in both cases. The architecture is implemented for the obstacle problem with variable obstacles~\cite[Example 4.4, Section 6.3]{schwab}, where $\obs$ can be mapped to the solution of the problem. The architecture in our work is analyzed and implemented for a variable obstacle and an additional variable coefficient and forcing.

Here, a multi-level decomposition of the solutions is utilized to implement individual networks approximating a coarse solution and fine grid corrections.
As in this decomposition corrections on fine grids are only of small values, only a low accuracy is needed on fine grids with many parameters. This can be made use of by implementing comparably small NN architectures on high levels in terms of either the number of trainable parameters or number of samples on fine grids.
Stochastic properties of the parametric obstacle problem are computed and analyzed based on multi-level decompositions of finite element spaces in~\cite{kornhuber2014} for fixed obstacles and in~\cite{Bierig2015ConvergenceAO} for stochastic obstacles, based on adaptive finite element methods in~\cite{kornhuber2018}.

The considered architecture inspired by the CNN constructed in~\cite{JMLR:v24:23-0421} for parametric partial differential equations is analyzed in terms of expressivity specifically for the obstacle problem, i.e. with respect to the needed number of trainable parameters to achieve a required accuracy. We prove that CNNs can approximate a projected Richardson iteration leading to bounds on the number of parameters only depending logarithmically on the required accuracy. Furthermore, we prove that a multigrid algorithm based on the projected Richardson iteration and a monotone restriction operator can be approximated by the applied architecture. 
This shows that the surrogate model is at least as expressive as the multigrid solver.

The CNN architecture is tested for different parameter dimensions and for constant and variable obstacles and elasticities.

Main contributions:
\begin{itemize}
    \item A CNN architecture mapping the coefficient, obstacle, and force to the solution of the obstacle problem is presented.
    \item The architecture is analyzed in terms of expressivity. To the best of the authors knowledge, the achieved theoretical convergence results have so far only been derived for obstacle-to-solution maps for other architectures. The main result is presented in~\cref{theorem: full CNN from richardson}. Assuming that the coefficient is uniformly bounded from below and above there exists a constant $C>0$ such that for any $\varepsilon>0$ there exists a CNN $\Psi$ with the number of parameters bounded by $\# \Psi \leq C \log\left({\varepsilon^{-1}}\right)$ such that and for all $\bfy\in\Gamma$ parameterizing the coefficient, force and obstacle it holds 
    \begin{align*}
        \norm{\Psi (\bfkappa,\bff, \bfobs) - \bfu(\cdot,\bfy)}_{H^1} \leq C \left(\norm{f}_{\ast} + \norm{\bfobs}_{H^1}\right)\varepsilon,
    \end{align*}
    where $\bfu$ is the collection of finite element coefficients of the solution of a discretized parametric obstacle problem and $\bfkappa,\bff,\bfobs$ are the discretized coefficient, force, and obstacle.
    \item The combination of the provably well suited architecture and a multi-level decomposition of solutions as an output of the CNN leads to state-of-the-art numerical results.
\end{itemize}

\section{Preliminaries}\label{section: fem}
Throughout this work let $D\subset \mathbb{R}^d$ be a domain with a smooth boundary and $\Gamma\subset \mathbb{R}^{\mathbb{N}}$ be a countably infinite dimensional parameter space. Furthermore, let
$\obs:D\times\Gamma\to\mathbb{R}$ be a smooth obstacle such that $\obs(x,\bfy) \leq 0$ for all
$x\in\partial D,\bfy\in\Gamma$. 
Let $\kappa: D\times \Gamma \to \mathbb{R}$ be a coefficient,
which is uniformly bounded from above and below, i.e. there exist a constants $\lb, \ub>0$ such that for all $\bfy\in\Gamma$ 
and $x\in D$ it holds $\lb \leq \kappa(x,\bfy) \leq \ub$. 
This implies uniform ellipticity of the differential operator.
Let $f: D\times\Gamma \to\mathbb{R}$ be the forcing such that $f(\cdot,\bfy)\in L^2(D)$ for each $\bfy\in\Gamma$.

For $v\in H_0^1(D)$, we make use of the norms
\begin{align*}
    \norm{v}_{L^2(D)}^2 &\coloneqq \int_D v^2 \dx x,\\
    \norm{v}_{H^1(D)}^2 &\coloneqq \norm{v}_{L^2(D)}^2 + \norm{\nabla v}_{L^2(D)}^2,\\
    \norm{v}_{A_\bfy}^2 &\coloneqq \int_D \kappa(\cdot,\bfy) \scpr{\nabla v, \nabla v} \dx x.\\
\end{align*}
For $v_2\in L^2(D)\hookrightarrow H^{-1}$, we associate $v_2$ with its associated function in $H^{-1}$ to define the dual norm by
\begin{align}\label{eq: dual norm}
    \norm{v_2}_{H^{-1}} \coloneqq \sup_{\substack{v\in H^1\\ \norm{v}_{H^1}=1}} \int_D v_2(x) v(x) \dx{x}.
\end{align}

Generating training samples for our approach relies on numerical methods for solving problem~\eqref{equation: VI}.
As our method is heavily inspired by multigrid solvers, we now introduce the finite element based methodology usually applied to this type of problem.
Here solutions are approximated in finite dimensional subspaces of $H_0^1(D)$. 
It leads to an algebraic equation to identify the coefficients of a linear combination of basis functions determined by a mesh. 
In this work, classical P$1$ finite element spaces are considered. We skip well-known results about standard FEM and refer to~\cite{verfuerth,elman,braess} for a detailed overview.

In our setting, let $\T$ be a uniform triangulation of the domain $D$ with nodes $\mathcal{N}$. For the number of nodes in the triangulation $N\coloneqq\abs{\mathcal{N}}$ and each node $i=1,\dots,N$ let $\lambda_i$ be the nodal hat function, which is linear on every triangle, equal to $1$ at node $i$ and $0$ at every other node. The set of all such hat functions is then referred to a s the P$1$ FE basis. Let $h>0$ be the minimal side length over all triangles in $\T$. The considered finite element space is defined by $\fesp \coloneqq \spa \{\lambda_i : i=1,\dots N \}$. Functions $v_h\in\fesp$ can then be written as linear combination of basis functions in the form $v_h = \sum_{i=1}^N \bfv_i\lambda_i$ with coefficients $\bfv\in\mathbb{R}^{N}$.

Note that the uniform lower and upper bounds on $\kappa$ imply the existence of constants $c_{H^1},C_{H^1}, \chl>0$ such that for all $v_h\in V_h$ 
it holds that
\begin{align}
    c_{H^1}\norm{v_h}_{A_\bfy} \leq \norm{v_h}_{H^1(D)} &\leq C_{H^1} \norm{v_h}_{A_\bfy},\label{eq:norm equivalence A and H1}\\
    \norm{\scpr{\nabla v_h, \nabla v_h}}_{L^2(D)} &\leq \chl \norm{v_h}_{L^2(D)}.\label{eq: reverse Poincare}
\end{align}
The second equation is often called \emph{reverse Poincaré inequality}. 

\section{Parametric obstacle problem}
The considered \emph{obstacle problem} is described in~\eqref{equation: VI}. In our parametric setting, the obstacle, the coefficient and the forcing depend on some possibly countably infinite dimensional parameter vector $\bfy\in\Gamma\subset\mathbb{R}^{\mathbb{N}}$. This assumption is opposed to $\bfy$ only parameterizing the obstacle as for instance implemented in~\cite{schwab}.

Note that the contact set, i.e. the area where the solution is equal to the obstacle, is not known in advance. An example of a parameter field sample, the solution and the contact set is depicted in \Cref{fig:obstacle}.
\begin{figure}
    \centering
    \includegraphics[width=0.32\linewidth]{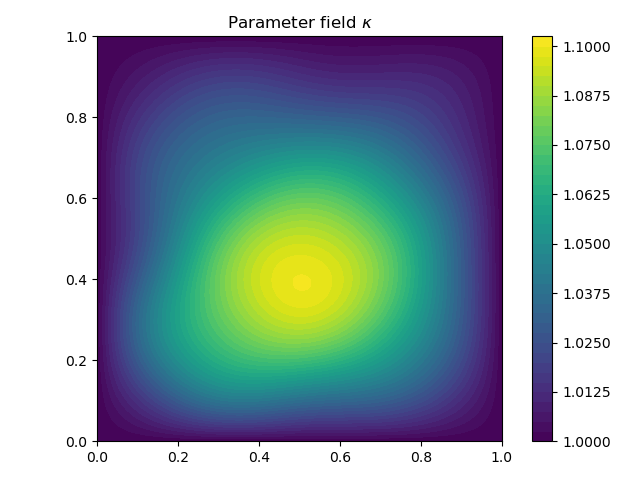}
    \includegraphics[width=0.34\linewidth]{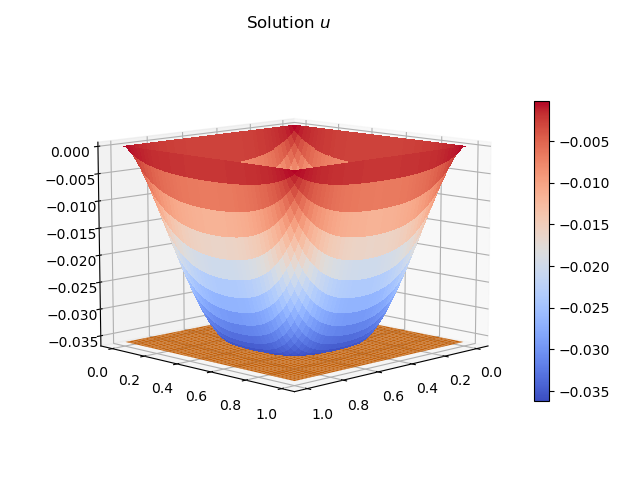}
    \includegraphics[width=0.32\linewidth]{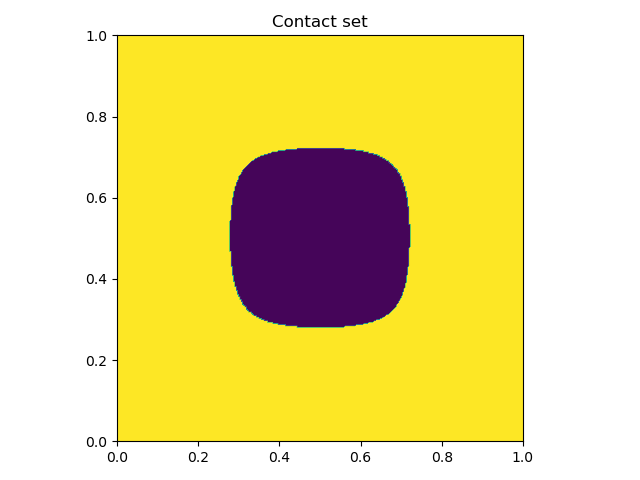}
    \caption{An example realization of a field $\kappa$, the respective solution to the obstacle problem $u$ and the corresponding contact set indicating where the solution is equal to the obstacle are shown for a constant obstacle $\obs \equiv -0.036$. The solution is equal to the obstacle in the purple part in the last image while it satisfies the PDE on the yellow part of the domain. Since the contact set is unknown in advance, it is part of the solution for the given parameter field.}
    \label{fig:obstacle}
\end{figure}

For the finite element approach, the problem is expressed in terms of a variational formulation.
Discretizing the obstacle $\obs(\cdot,\bfy)$, the ellipticity $\kappa(\cdot,\bfy)$ in $V_h$, testing the forcing $f$ in $V_h$, and discretizing the test functions $v$ and the solution $u$ in the set
$K \coloneqq \{v_h\in V_h : v_h\geq \obs \text{ a.e. in } D\}$
the following variational formulation can be derived.

\begin{problem}[Variational parametric obstacle problem]\label{problem: variational obstacle discretized}
    Find $u_h\in K$ such that for all $v_h\in K$ it holds
    \begin{align*}
        \int_D \kappa(\cdot,\bfy)\scpr{\nabla u_h, \nabla(v_h-u_h)} \dx{x} \geq \int_D f(x)(v_h-u_h)(x)\dx x.
    \end{align*}
\end{problem}
With uniform ellipticity as assumed here, it can be shown that a unique solution of~\eqref{problem: variational obstacle discretized} exists for every $\bfy\in\Gamma$, e.g. see~\cite[Chapter 2.2]{glowinski} for a proof. 
Note that choosing $V_h$ to be the finite element space of low order polynomials, e.g. here $P1$ basis functions, is sufficient. This is a result of the solution of the nonlinear problem in general not being "very smooth" over the whole domain despite smooth data, e.g. see~\cite[Chapter 5]{ciarlet}.
In terms of the finite element coefficients the problem is equivalently written in the following form, e.g. see~\cite[Chapter 2.3]{glowinski}.

\begin{problem}[Variational parametric obstacle problem, discretized] \label{problem: parametric obstacle discretized}
    Let $A_\bfy\in\mathbb{R}^{N\times N}$ be the discretized operator and $\bff\in\mathbb{R}^{ N}$ the tested forcing defined for $i,j=1,\dots,N$ by 
    \begin{align}\label{equation: Ay}
        (A_\bfy)_{i,j} = \int_D \kappa_h(x,\bfy) \scpr{ \nabla \lambda_i(x), \nabla \lambda_j(x)} \dx{x} \quad \text{ and } \quad \bff_i = \int_D f(x)\lambda_i(x)\dx x.
    \end{align}
    Find $\bfu \in\mathbb{R}^{N}$ such that for the discretized obstacle $\obs_h = \sum_{i=1}^N \bfobs_i\lambda_i$ it holds
    \begin{align}\label{equation: obstacle discretized cases}
        \begin{cases}
            \quad \bfu_i \geq \bfobs_i &\quad\text{for } i=1,\dots,N,\\
            \quad (A_\bfy\bfu)_i \geq \bff_i &\quad\text{for } i=1,\dots,N,\\
            \quad (A_\bfy\bfu)_i = \bff_i &\quad\text{for } \bfu_i > \bfobs_i.
        \end{cases}
    \end{align}
\end{problem}
Since $\obs$ is not constrained to be zero on the boundary, one has to apply the discretization of the obstacle with care. The utilization fo P$1$ elements here circumvents this consideration as the obstacle condition can equivalently be enforced on inner vertices.
Additionally, note that it holds
\begin{align}\label{eq: A_y norm equal}
    \norm{v_h}_{A_\bfy}^2 = \sum_{i,j=1}^N \bfv_i\bfv_j \int_D \kappa(\cdot,\bfy) \scpr{\nabla \lambda_i,\nabla\lambda_i}\dx x = \bfv^\intercal A_\bfy \bfv \eqqcolon \norm{\bfv}_{A_\bfy}^2.
\end{align}

\section{Multigrid solver}\label{sec: multigrid solver}

In this work the focus lies on solving the presented discretized obstacle problem~\eqref{equation: obstacle discretized cases} with a suitable NN architecture, which provides practical and theoretical benefits. For the theoretical underpinnings it is common to analyze NNs with respect to the number of trainable parameters as a measure of representation complexity. 
In the next chapters it is shown that our NN architecture is able 
to approximate a classical constrained multigrid solver. This means that the network is at least as expressive as multigrid solvers with the possibility to find more accurate solutions. 
The central property we require for the later analysis is that the algorithm exhibits a structure amenable to an efficient NN approximation.
This then allows the derivation of a quantitative convergence guarantee with complexity bounds.
The solver provably converges and is based on a multigrid algorithm with a projection method as a smoother on every grid as detailed subsequently.

\subsection{Projected Richardson iteration}\label{section: projected richardson}

Numerous algorithms have been developed in the past decades to solve \cref{problem: variational obstacle discretized} many of which are iterative approaches, see e.g.~\cite{solverforobstacle,glowinski,ZHANG20011505}. As a preparation for the later CNN construction, a projection method related to the Richardson iteration is introduced in the following. The particular version considered here can be found in~\cite[Section 3]{ZHANG20011505} and is called \emph{projected Richardson iteration} throughout this work.

The main idea is to iteratively update an approximate solution with a weighted residual and apply a projection operation to enforce the given obstacle constraint. Suppressing the dependence on $\bfy\in\Gamma$ in the notation, let $\omega>0$ be some damping parameter, $A\in\mathbb{R}^{N\times N}$ be defined as in~\eqref{equation: Ay}, $\bff$ the tested right-hand side and $\bfobs$ be the finite element coefficients of the obstacle of the parametric obstacle problem. Then, the algorithm consists of iterating dampened updates given by
\begin{align}
\begin{split}
    \bfu^{(0)} &\coloneqq \bfobs,\\
    \bfu^{(k+1)} &\coloneqq \max\left\{ \bfu^{(k)} + \omega \left( \bff - A \bfu^{(k)} \right), \bfobs \right\}, 
\end{split}\label{equation: projected richardson}
\end{align}
where the maximum is to be understood component-wise.
To analyze the convergence of the algorithm, the residual $\bfe^{(k)} \coloneqq \bfu^{(k)} - \bfu$ is considered, where $\bfu$ is the solution of~\cref{problem: parametric obstacle discretized}.
With~\eqref{equation: obstacle discretized cases}, the solution satisfies 
\begin{align*}
    \max\left\{ \bfu + \omega(\bff - A \bfu), \bfobs \right\} = \bfu + \max\left\{ \omega(\bff - A \bfu), \bfobs  -\bfu \right\} = \bfu.
\end{align*}
This can be seen by considering that for $i=1,\dots,N$ on the one hand $\bfobs_i - \bfu_i <0$ implies that $(\bff - A\bfu)_i = 0$ and therefore the maximum is zero. On the other hand, $\bfobs_i-\bfu_i=0$ implies that $(\bff - A\bfu)_i\leq 0$ also leading to a maximum of zero.
Using that the mapping $\bfx\mapsto \max\{\bfx,\bfobs\}$ is a contraction, the residual can be bounded as follows.
\begin{align}\label{equation: Richardson decay}
\begin{split}
    \norm{\bfe^{(k+1)}}_A 
    &= \norm{\bfu^{(k+1)} - \bfu}_A\\
    &= \norm{\max\left\{ \bfu^{(k)} + \omega\left(\bff - A \bfu^{(k)}\right), \bfobs \right\} - \max\left\{ \bfu + \omega(\bff - A \bfu), \bfobs \right\}}_A\\
    &\leq \norm{\bfu^{(k)} + \omega\left(\bff - A \bfu^{(k)}\right) - (\bfu + \omega(\bff - A \bfu))}_A\\
    &= \norm{\bfe^{(k)} + \omega A(\bfu - \bfu^{(k)})}_A\\
    &\leq \norm{I - \omega A}_A \norm{\bfe^{(k)}}_A.
\end{split}
\end{align}
Therefore, the rate of convergence of the method is bounded by the energy norm of $I - \omega A$, where $\omega>0$ needs to be chosen appropriately to ensure a contraction.
For $A \coloneqq A_\bfy$ defined as in \eqref{equation: Ay}, we choose and bound $\omega$ independently of $\bfy$ to ensure convergence of the projected Richardson iteration for any $\bfy\in\Gamma$. First, $\omega_\bfy$ is chosen dependent on $\bfy$ such that the norm is bounded by a constant smaller than $1$ also depending on $\bfy$.

\begin{lemma}[generalization of {\cite[Lemma B.2]{schutte2024multilevelcnnsparametricpdes} or \cite[Lemma 4.3]{braess}}]\label{lm: subspace contraction}\label{theorem:omega>gamma}
    Let $\bfy\in\Gamma$ and $\kappa(\cdot,\bfy)>0$ everywhere.
    Then for any nonzero $\bfw\in\mathbb{R}^{N}$ it holds that
    \begin{align*}
        \norm{(I - \omega_\bfy A_\bfy)\bfw}_{A_\bfy} \leq  (1-\omega_\bfy) \norm{\bfw}_{A_\bfy},
    \end{align*}
    where $0<\omega_\bfy\leq \sigma_{\max}(A_\bfy)^{-1}$. 
\end{lemma}
The proof can be found in \cref{section: proof of CNN for obstacle}. 
Second, $\omega$ is chosen independently of $\bfy$ such that the operator norm is bounded by a constant smaller than $1$, which is also independent of $\bfy$.
\begin{lemma}\label{lemma:gamma}
    Assume that $\kappa$ is uniformly bounded, i.e. there exists a constant $\ub>0$ such that $\kappa(x,\bfy)\leq \ub$ for all $x\in D, \bfy\in\Gamma$. Then, for all $0<\omega\leq\frac{1}{\ub\chl}$ and $\bfy\in\Gamma$ it holds $\norm{I - \omega A_\bfy}_{A_\bfy} \leq 1-\omega$.
\end{lemma}
\begin{proof}
    First, we note that the maximal eigenvalue of $A_\bfy$ is bounded as can be derived as follows. Let $\bfv\in\mathbb{R}^{\dim V_h}$ be an eigenvector of $A_\bfy$ corresponding to the maximal eigenvalue and $v_h\in V_h$ be the corresponding finite element function. Then, for the equivalence constant $\chl$ of the inverse Poincare inequality 
    in the finite dimensional space $V_h$ it holds that
    \begin{align*}
        \sigma_{{\max}}(A_\bfy) = \frac{\bfv^T A_\bfy \bfv}{\bfv^T\bfv} = \frac{\int_D \kappa_h(\cdot,\bfy) \scpr{\nabla v_h, \nabla v_h} \dx{x}}{\int_D v_h^2 \dx{x}} \leq \ub \frac{\int_D \scpr{\nabla v_h, \nabla v_h} \dx{x}}{\int_D v_h^2 \dx{x}} \leq \ub \chl^2.
    \end{align*}
    Choosing $0 <\omega \leq (\ub \chl^2)^{-1} \leq \sigma_{\max}(A_\bfy)^{-1}$ for all $\bfy\in\Gamma$ yields the claim with \cref{theorem:omega>gamma}.
\end{proof}

\subsection{Geometric Multigrid}\label{section:geometricMultigrid}
Note that the convergence rate $1-\omega$ in \cref{lm: subspace contraction} is close to one for small $\omega$, i.e. for large constants $\chl$ defined in~\eqref{eq: reverse Poincare}.
In case of a uniform triangulation as considered here it can be shown that $\chl = C h^{-1}$ is a possible choice for some constant $C>0$ only depending on the angles of the triangles, see e.g.~\cite[Lemma 1.26]{elman} or~\cite{thomee}.
In the two dimensional setting with a uniform triangulation as considered here, the nodes are arranged on a uniform grid. Therefore, the minimal side length of all triangles is given by $h = (\sqrt{N}-1)^{-1}$
leading to $C_{V_h} = C (\sqrt{N}-1)$.
\cref{lm: subspace contraction} therefore yields $\omega \lesssim N^{-1}$. 
This relationship implies slow convergence of the projected Richardson iteration for high fidelity discretizations, which is a reason why it cannot be considered state-of-the-art when solving discretized PDEs.

Instead, the projected Richardson iteration provides the basis for geometric multigrid methods, which are able to efficiently solve the problem at hand \cite{solverforobstacle,kornhuber2002}.
An interplay between different meshes can lead to a speedup in convergence with a number of necessary iterations independent of the grid fidelity. Such results have been shown for instance for discrete Poisson problems in~\cite[Theorem 2.14]{elman} and \cite[Theorem 4.2]{braess}.

Multigrid methods are based on a set of $L\in\mathbb{N}$ triangulations $\T_1,\dots,\T_L$, e.g. generated by a uniform or adaptive mesh refinement starting on the coarsest mesh $\T_1$, with nodes $\N_1,\dots,\N_L$ and $N_\ell \coloneqq\abs{\N_\ell}$ for $\ell=1,\dots,L$ such that the corresponding subsequent finite element spaces are nested
\begin{align*}
    V_1\subset \dots \subset V_L \subset H_0^1(D).
\end{align*}
The method then projects approximate solutions to finer spaces or restricts them to coarser spaces, applying smoothing iterations (the projected Richardson iteration) on coarse and fine grids successively. 
The considered prolongation operator used to interpolate functions on coarse grids into spaces on fine grids in the discretized setting is defined as follows.
\begin{definition}[Prolongation matrices]
    Let $L\in\mathbb{N}$ be the number of grids and for some $\ell\in\{1,\dots, L\}$ let $V_\ell$ and $V_{\ell+1}$ as above. Then, the \emph{prolongation matrix} $P_\ell\in\mathbb{R}^{N_{\ell+1}\times N_\ell}$ is the matrix representation of the canonical embedding of $V_\ell$ into $V_{\ell+1}$ under their respective finite element basis functions.
\end{definition}
The considered restriction operator maps coefficients on a fine grid to coefficients on a coarse grid such that the obstacle condition is still satisfied. The problem of applying the the restriction used in~\cite{JMLR:v24:23-0421} to the obstacle problem is shown in~\Cref{fig:restriction}.
\begin{definition}[Monotone restriction operator]\label{definition: Rl}
    For $\ell=1,\dots,L-1$ let $V_{\ell} \coloneqq \spa \{\lambda_i^{(\ell)}\}_{i=1}^{N_\ell}$ and $V_{\ell+1} \coloneqq \spa \{\lambda^{(\ell+1)}_i\}_{i=1}^{N_{\ell+1}}$ be two nested P$1$ finite element spaces as above. Then, define the \emph{monotone restriction operator} $\Rl:\mathbb{R}^{N_{\ell+1}} \to \mathbb{R}^{N_{\ell}}$ by
    \begin{align*}
        (\Rl \bfu)_i = \max \left\{ \bfu_j: \supp \lambda_j^{(\ell+1)} \subset \supp \lambda_i^{(\ell)} \right\}.
    \end{align*}
\end{definition}
\begin{figure}
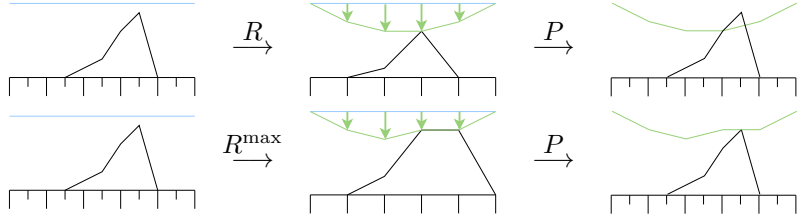

    \centering
    \begin{tikzpicture}
        \node at (0,0) {\includesvg[width=0.15\textwidth]{images/RestrictionOperator_fine.svg}};
        \node at (4,0) {\includesvg[width=0.15\textwidth]{images/RestrictionOperator_wrong_coarse.svg}};
        \node at (8,0) {\includesvg[width=0.15\textwidth]{images/RestrictionOperator_wrong_fine.svg}};
        \node at (2,0) {$\longrightarrow$};
        \node[above] at (2, 0) {$R$};
        \node at (6,0) {$\longrightarrow$};
        \node[above] at (6, 0) {$P$};

        \node at (0,-1.5) {\includesvg[width=0.15\textwidth]{images/RestrictionOperator_fine.svg}};
        \node at (4,-1.5) {\includesvg[width=0.15\textwidth]{images/RestrictionOperator_right_coarse.svg}};
        \node at (8,-1.5) {\includesvg[width=0.15\textwidth]{images/RestrictionOperator_right_fine.svg}};
        \node at (2,-1.5) {$\longrightarrow$};
        \node[above] at (2, -1.5) {$R^{\max}$};
        \node at (6,-1.5) {$\longrightarrow$};
        \node[above] at (6, -1.5) {$P$};
    \end{tikzpicture}
    \caption{The first row images show the weighted restriction as defined in \cite{JMLR:v24:23-0421}. A visualization of the restriction operator defined in \cref{definition: Rl} is depicted in the second row images. In both rows the first image illustrates an obstacle in black and an initial guess for the solution in blue. The second images show the restricted obstacle together with a coarse grid solution in green. The last images depicts the prolongated coarse grid solution together with the true obstacle. It can be seen that taking a maximum, when restricting the obstacle, is critical for the coarse grid solution to still be above or equal to the true obstacle on the finer grid. The dependence on a level $\ell$ is suppressed in the notation. }
    \label{fig:restriction}
\end{figure}

These operators are used in the multigrid V-Cycle with monotone restriction ($\mathrm{VCMR}$) in~\cref{alg: vcmr}. Starting on the finest grid $\ell=L$, the algorithm performs $k$ projected Richardson iterations on the given grid. Subsequently, the problem is projected to a coarser grid to approximate a correcting term, where computations are cheaper. To approximate this term, the $\mathrm{VCMR}$ is called again with inputs restricted to the next coarser grid $\ell-1$. The obstacle is restricted by the monotone restriction operator and the residual and the operator are restricted by the transposed prolongation (or weighted restriction) operator. On the coarsest level $\ell=1$, the solution to the input problem is computed directly, e.g. by applying projected Richardson iterations until the algorithm has converged.
The correction terms are returned to the higher levels and added to the current solution approximations. After the correction is added, another $m$ smoothing steps are performed. 
The notation $\vcmrm$ is used to describe the application of the $\mathrm{VCMR}_{k,\ell}$ $m\in\mathbb{N}$ times with $k_0$ smoothing steps on the coarsest grid.

Empirically, the described multigrid method provides a significant speedup. While for some problems the speedup can be quantified theoretically~\cite{elman,braess}, to the best of our knowledge this has not been shown for the obstacle problem. 
In contrast to multigrid methods for PDEs the nonlinearity of the obstacle introduces errors in the coarse grid corrections through the monotone restriction, which can slow down asymptotic convergence rates as for instance described in~\cite[Section 5.2]{solverforobstacle}. 
The reason for using the projected Richardson iteration on every grid, despite it not being a state-of-the-art method (see~\cite{solverforobstacle,ZHANG20011505}), is its simplicity and the implications for the neural network architecture analyzed in the following chapters.

\begin{algorithm}
    \caption{Multigrid V-Cycle with monotone restriction: $\mathrm{VCMR}_{k,\ell}$}\label{alg: vcmr}
    \textbf{Input:} $\bfu, \bff, A_\bfy, \bfobs$\\
    \For{$k$ pre-smoothing steps}{
    $\bfu \leftarrow \max\left\{\bfu + \omega(\bff - A_\bfy \bfu), \bfobs \right\}$ \Comment{Perform smoothing steps}
    }
    \If{$\ell=1$}{
    solve $A_\bfy \bfu = \bff$ for $\bfu$ on coarsest level \Comment{E.g. by smoothing until converged}
    }
    \Else{
    $\overline{\bfobs} \leftarrow \Rl (\bfobs - \bfu) $ \Comment{Compute monotone restricted obstacle}\\
    $\overline\bfr \leftarrow P_{\ell-1}^\intercal(\bff - A_\bfy \bfu)$ \Comment{Compute restricted residual}\\
    $\overline A_\bfy \leftarrow P_{\ell-1}^\intercal A_\bfy P_{\ell-1}$ \Comment{Compute restricted operator}\\
    $\overline{\bfe}\leftarrow$ $\mathrm{VCMR}_{k,\ell-1}(0, \overline{\bfr}, \overline{A_\bfy}, \overline{\bfobs})$ \Comment{Use V-Cycle with monotone restriction on coarser grid}\\
    $\bfu \leftarrow \bfu + P_{\ell-1}\overline{\bfe}$ \Comment{Add coarse correction}
    }
    \For{$k$ post-smoothing steps}{
    $\bfu \leftarrow \max\left\{\bfu + \omega(\bff - A_\bfy \bfu), \bfobs \right\}$ \Comment{Perform smoothing steps}
    }
    \textbf{return:} $\bfu$
\end{algorithm}

\section{Convolutional neural network}

Convolutional neural networks (CNNs) have been proven to be an efficient tool for approximating solutions to parametric PDEs, see~\cite{JMLR:v24:23-0421, schutte2024adaptive, schutte2024multilevelcnnsparametricpdes}.
For the application of this architecture to the parametric obstacle problem, some additional steps have to be taken.
For the analysis, the following conditions are assumed for the activation function throughout this paper.
\begin{assumption}[Activation function]\label{ass: sigma}
    Let $\tau:\mathbb{R}\to\mathbb{R}$ satisfy the following conditions:
    \begin{enumerate}
        \item There exists $x_0\in\mathbb{R}$ and an open interval $I\subset\mathbb{R}$ with $x_0\in I$ such that $\tau$ is three times differentiable on $I$ and $\tau''(x_0) \neq 0$.
        \item For any $\varepsilon>0$ there exist two affine-linear mappings $\rho_1,\rho_2:\mathbb{R}\to \mathbb{R}$ such that for all $x\in\mathbb{R}$
        \begin{align*}
            \abs{(\rho_2\circ\tau\circ\rho_1)(x) - \max\{x,0\}} \leq\varepsilon.
        \end{align*}
    \end{enumerate}
\end{assumption}
These conditions are fulfilled by a family of soft ReLU variants such as Softplus, SeLU, or ELU, see~\cite{ReLUvariants}. The first condition is needed in the expressivity analysis to approximate multiplication, while the second condition is applied to approximate the maximum in the projected Richardson iteration. In the expressivity theory in~\cite{JMLR:v24:23-0421}, only the first condition is necessary since the maximum does not need to be approximated. The number of trainable parameters of a given CNN $\Psi$ is denoted by $M(\Psi)$.

\subsection{Expressivity}

Similar to \cite[Theorem 6]{JMLR:v24:23-0421} an approximation of a multigrid V-Cycle can be can be imitated by a CNN. Here the $\vcmrm$ is approximated with projected Richardson iteration instead of a plain Richardson iteration and newly including the monotone restriction operator.
For the expressivity analysis, it is first shown that the projected Richardson iteration can be approximated by a CNN of size growing linearly in the number of iterations.

\begin{theorem}[CNN for the projected Richardson iteration]\label{theorem: CNN for richardson}
    There exists a constant $C>0$ such that for any $\omega,M,\varepsilon>0$ and $m\in\mathbb{N}$ there exists a CNN $\Psi:\mathbb{R}^{4\times n\times n} \to \mathbb{R}^{1\times n\times n}$ such that for all discretized coefficients $\bfkappa \in [-M,M]^{n\times n}$, initial guesses $\bfu\in[-M,M]^{n\times n}$, obstacles $\bfobs\in [-M,M]^{n\times n}$ and tested right-hand sides $\bff\in[-M,M]^{n\times n}$ it holds
    \begin{enumerate}[label=(\roman*)]
        \item $\norm{\Psi(\bfu,\bfkappa,\bff,\bfobs) - \bfu^{(m)}}_{H^1(D)} \leq \varepsilon$,
        \item number of parameters bounded by $M(\Psi) \leq Cm$.
    \end{enumerate}
\end{theorem}
The proof can be found in~\cref{section: proof of CNN for obstacle}.
As derived in~\cref{section: projected richardson}, the projected Richardson iteration can approximate the true solution up to arbitrary accuracy if the smoothing coefficient $\omega>0$ is chosen appropriately. 
Then, the CNN approximating the Richardson iteration also approximates the solution of problem~\cref{problem: parametric obstacle discretized} as shown in the next corollary.

\begin{corollary}[CNN for parametric obstacle problem]\label{theorem: full CNN from richardson}
    Assume that $\lb \leq \kappa(x,\bfy)\leq\ub$ is uniformly bounded by some $\lb,\ub>0$ over all $\bfy\in\Gamma, x\in D$. Let $\bfu(\cdot,\bfy, \bfobs)$ denote the solution of \cref{problem: parametric obstacle discretized} for the parameter vector $\bfy\in\Gamma$ and obstacle $\bfobs$. For every $\varepsilon>0$, there exists a CNN $\Psi$ such that 
    \begin{enumerate}[label=(\roman*)]
        \item for all $\bfy\in\Gamma$ it holds that \begin{align*}
        \norm{\Psi (\bfobs, \bfkappa_\bfy,\bff, \bfobs) - \bfu(\cdot,\bfy, \bfobs)}_{H^1} \leq \varepsilon \left( C^2_{H^1}\norm{f}_{H^{-1}} + \frac{C_{H^1}}{c_{H^1}} \norm{\bfobs}_{H^1}\right),
        \end{align*}
    \item the number of parameters is bounded by $\# \Psi \leq C \left\lceil {\log\left({\varepsilon^{-1}}\right)\log \left({1-(\ub\chl)^{-1}}\right)^{-1}}\right\rceil$, where $C>0$ is the independent constant from \cref{theorem: CNN for richardson}.
    \end{enumerate}
\end{corollary}
The proof can be found in~\cref{section: proof of CNN for obstacle}.
The upper bound for number of parameters is grid-dependent through the constant $\chl$. In~\cite{JMLR:v24:23-0421} the dependence on the size $h$ is circumvented by approximating a multigrid algorithm based on the Richardson iteration, which inspires a specific CNN architecture. Here, a fitting multigrid is approximated as well.
The approximation of the Richardson iteration in~\cref{theorem: CNN for richardson} with the approximation of the prolongation and monotone restriction operator shows that the whole~\cref{alg: vcmr} can be approximated by a U-net based CNN similar to the construction in~\cite{JMLR:v24:23-0421}. The main differences lie in the additional projection in each step of the Richardson iteration and in the monotone restriction operator. The approximation of the monotone restriction operator is proven in the appendix in~\cref{theorem: CNN for monotone restriction}.

\begin{remark}[CNN for multigrid algorithm]\label{theorem: CNN for V-Cycle}
    Let $V_h \subset H_0^1(D)$ be the P1 FE space on a uniform square mesh. Then there exists a constant $C > 0$ such that for any $M, \varepsilon > 0$ and $k,m,\ell \in \mathbb{N}$ there exists a CNN
    $\Psi : \mathbb{R}^{4\times \dim V_h} \to \mathbb{R}^{\dim V_h}$ with
    \begin{enumerate}[label=(\roman*)]
        \item $\norm{\Psi(\bfu_0, \bfkappa,\bff, \bfobs) - \vcmrm(\bfu_0, \bfkappa,\bff, \bfobs)}_{H^1(D)} \leq \varepsilon$ for all $\bfu_0, \bfkappa,\bff, \bfobs \in [-M,M]^{\dim V_h}$
        \item number of weights bounded by $M(\Psi) \leq C m \ell$.
    \end{enumerate}
\end{remark}
The proof can be found in \cref{section: proof of CNN for obstacle}. The analysis works similarly to the analysis carried out in ~\cite{JMLR:v24:23-0421}. The main difference lies in the additional application of the maximum of the solution approximation and the obstacle in every step of the projected Richardson iteration and in the restriction operator.
The architecture of the CNN used in the numerical experiments in this work is inspired by the constructive proof of~\cref{theorem: CNN for V-Cycle}.
A description of the architecture can also be found in~\cite{JMLR:v24:23-0421}.
To the best of our knowledge, a convergence speed-up of the $\vcmrm$ compared to the projected Richardson iteration has not been shown theoretically despite a considerable numerical speed up.
Therefore, another approximation of the solution based on this algorithm is not included here.

\subsection{Multi-level advantage}\label{section: multilevel}
Multi-level machine learning algorithms include training neural networks on efficient decompositions of the data to improve performance. 
In~\cite{LYE_MISHRA_MOLINARO_2021} the generalization error for such a decomposition was analyzed.
With the grids introduced in~\cref{section:geometricMultigrid}, lower-level models approximate a coarse grid solution and higher-level models approximate high fidelity corrections.
The implemented network introduced in~\cite{JMLR:v24:23-0421} is based on this multi-level decomposition.
Here, a discretized solution operator $u_L:\Gamma\to V_L$ of the parametric obstacle problem is decomposed into components $u_\ell:\Gamma \to  V_\ell$ on the individual spaces $V_1,\dots,V_L$ by 
\begin{align*}
    u_L = \sum_{\ell=1}^L u_{\ell} - u_{\ell-1} = \sum_{\ell=1}^L v_\ell,
\end{align*}
where $u_{0}$ is set to zero and $v_\ell:\Gamma\to V_\ell$ are additive corrections on each level. This decomposition is visualized in~\Cref{fig: multilevel}. The individual parts of the solution are approximated separately by CNNs.
First, a normalized solution on a coarse grid is approximated with a CNN $\Psi_1$. Then, individual networks $\Psi_2\dots,\Psi_L$ are trained to approximate normalized corrections with some normalization constant $b_\ell>0$ on finer grids with some accuracy $\varepsilon_\ell>0$ with 
\begin{align*}
    \norm{\Psi_\ell - \frac{v_\ell}{b_\ell}}\leq \varepsilon_\ell.
\end{align*}
Then, the weighted sum $\Psi \coloneqq \sum_{\ell=1}^L b_\ell\Psi_\ell$ approximates the whole solution by 
\begin{align*}
    \norm{\Psi - u^L} = \norm{\sum_{\ell=1}^L b_\ell \left(\Psi_\ell - \frac{v_\ell}{b_\ell}\right)} \leq \sum_{\ell=1}^L b_\ell \norm{\Psi_\ell - \frac{v_\ell}{b_\ell}}.
\end{align*}
If the normalization constant is chosen as an operator norm of $v_\ell$ and setting $b_\ell\coloneqq \norm{v_\ell}_{L^p(\Gamma,L^2(D))}$, the inequalities \cite[Equation 6,10, Equation 7.6]{Bierig2015ConvergenceAO} imply that
\begin{align*}
    \norm{\Psi - u^L} = \sum_{\ell=1}^L \varepsilon_\ell h_\ell(\norm{f}_{L^2(\Gamma,L^2(D))} + \norm{\obs}_{L^2(\Gamma,H^2(D))}) \leq C_{f,\obs} \sum_{\ell=1}^L\varepsilon_\ell 2^{-\ell}
\end{align*}
holds, when considering that the multi-level decomposition applied in the given architecture yields $h_\ell \leq C2^{-\ell}$ for the maximal side length $h_\ell$ of triangles in $\mathcal{T}_\ell$. To achieve an overall accuracy $C_{f,\obs}\sum_{\ell=1}^L \varepsilon_\ell 2^{-\ell}\leq \varepsilon$, the accuracy of each sub-model on each level only has to satisfy
\begin{align}\label{equation:multilevel-error}
    \varepsilon_\ell \leq \frac{\varepsilon 2^\ell}{LC_{f,\obs}}.
\end{align}
Since functions in high fidelity spaces have more degrees of freedom, they are in general more difficult to approximate. The advantage of the decomposition then lies in the fact that the approximation in higher dimensional spaces is allowed to be exponentially worse. This can be either translated to less trainable parameters of the CNN or fewer expensive high fidelity solutions for training.

\section{Numerical Experiments}
For the numerical tests, the U-net based architecture described in~\cite{JMLR:v24:23-0421} and supported by the result in~\cref{theorem: CNN for V-Cycle} was used. Here, $L=7$ nested FE spaces are considered. A first CNN then approximates the solution of the obstacle problem on the coarse grid FE space $V_1$. Further individual networks then approximate the corrections of the solution on finer meshes in $V_2,\dots, V_7$ as visualized in \Cref{fig: multilevel} for the first four spaces.
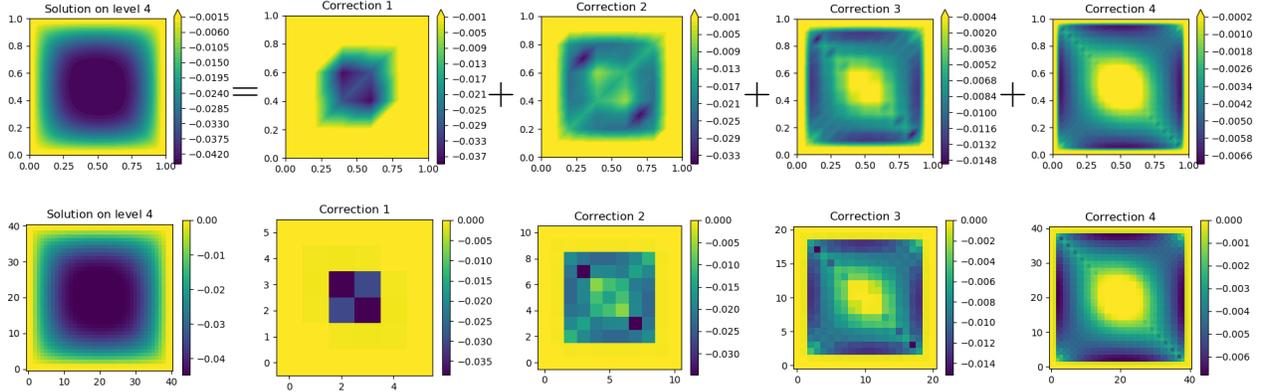
\begin{figure}
\hspace{-.5cm}
    \begin{tikzpicture}
        \def\imgsize{3.2cm}  
        \def\xgap{3.4}  
    
        \node (A1) at (0, 2.8) {\includegraphics[width=\imgsize]{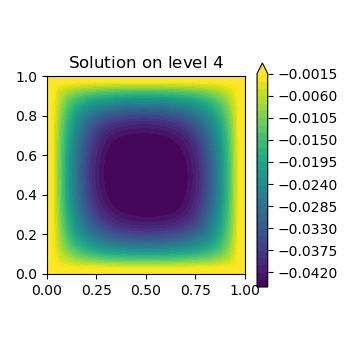}};
        \node (B1) at (\xgap, 2.8) {\includegraphics[width=\imgsize]{{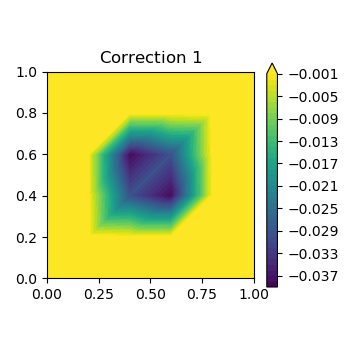}}};
        \node (C1) at (2*\xgap, 2.8) {\includegraphics[width=\imgsize]{{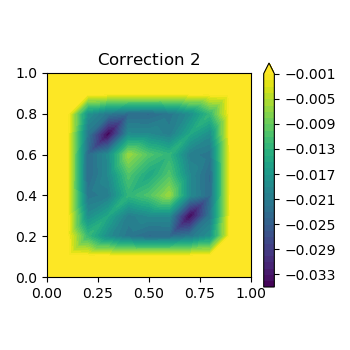}}};
        \node (D1) at (3*\xgap, 2.8) {\includegraphics[width=\imgsize]{{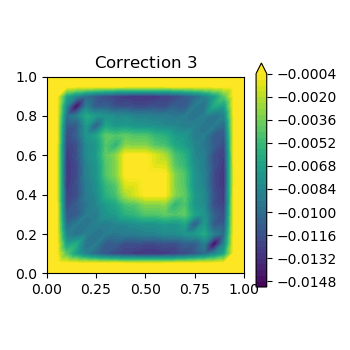}}};
        \node (E1) at (4*\xgap, 2.8) {\includegraphics[width=\imgsize]{{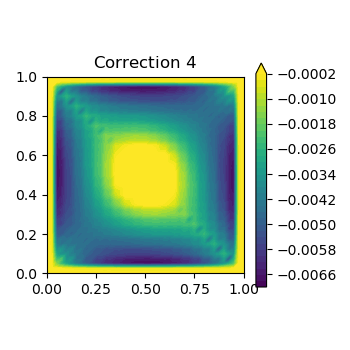}}};
    
        \node at (0.5*\xgap, 2.7) {\Large $=$};
        \node at (1.5*\xgap, 2.7) {\Large $+$};
        \node at (2.5*\xgap, 2.7) {\Large $+$};
        \node at (3.5*\xgap, 2.7) {\Large $+$};
    
        \node (A2) at (0, 0) {\includegraphics[width=\imgsize]{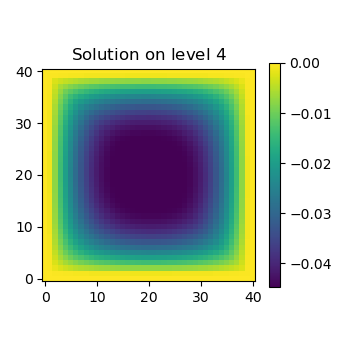}};
        \node (B2) at (\xgap, 0) {\includegraphics[width=\imgsize]{{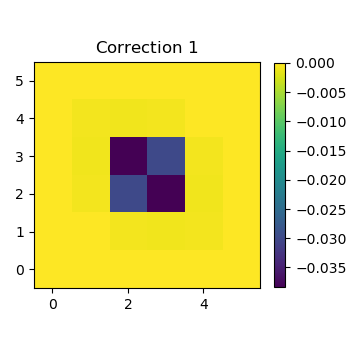}}};
        \node (C2) at (2*\xgap, 0) {\includegraphics[width=\imgsize]{{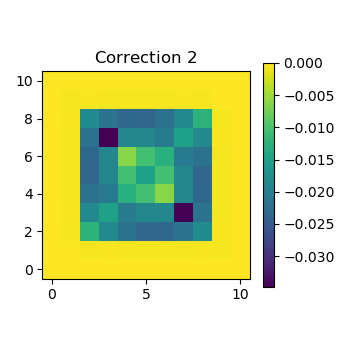}}};
        \node (D2) at (3*\xgap, 0) {\includegraphics[width=\imgsize]{{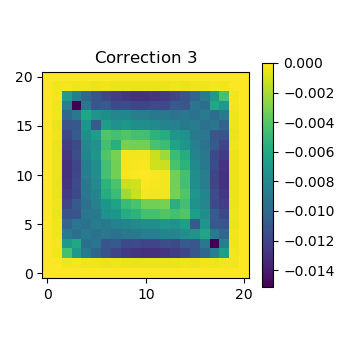}}};
        \node (E2) at (4*\xgap, 0) {\includegraphics[width=\imgsize]{{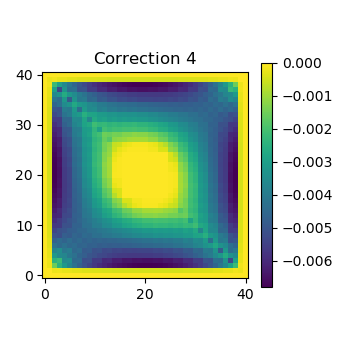}}};
    
    \end{tikzpicture}
    \caption{The solution to the obstacle problem in $V_4$ in the first row on the left-hand side can be decomposed into corrections in $V_1,V_2,V_3,V_4$ of decreasing values on different grids as seen in the first row. The sum of the corrections equals the full solution. The FE coefficients of the solution in $\mathbb{R}^{N_4}$ and the corrections in $\mathbb{R}^{N_1},\mathbb{R}^{N_2},\mathbb{R}^{N_3},\mathbb{R}^{N_4}$ are visualized in images underneath each function.}
    \label{fig: multilevel}
\end{figure}

The following test cases are considered.

\begin{enumerate}
    \item \textbf{Deterministic obstacle.} To solve \cref{problem: variational obstacle discretized} for a constant obstacle $\obs(x,\bfy)= -0.035$ and $f(x) = 1$ for all $x\in D = [0,1]^2$, a sample $\bfy\in\Gamma$ was drawn. As e.g. in~\cite{Eigel_2019, JMLR:v24:23-0421}, the coefficient field is assumed to have the representation
    \begin{align*}
        \kappa(x,\bfy) \coloneqq a_0(x) + \sum_{m=1}^p \bfy_m a_m(x),
    \end{align*}
    where $a_m(x) \coloneqq m^{-2} \sin(\lfloor\frac{m+2}{2}\rfloor\pi x_1)\sin(\lceil\frac{m+2}{2}\rceil\pi x_2)$ and $y$ is chosen uniformly in $\Gamma = [-1,1]^p$. A realization of the coefficient, the obstacle and the corresponding solution as well as the contact domain are visualized in~\Cref{fig:obstacle}.
    \item \textbf{Stochastic constant obstacle.} The problem above is now implemented with the additional variation of the obstacle with $\obs$ chosen as a constant function with value distributed uniformly in $[-0.045, -0.025]$. Since the value of the obstacle is the $p$-th entry of the $\bfy$, the sum above defining $\kappa$ only goes to $p-1$.
    \item \textbf{Rough surface obstacle.} Similar to \cite{Bierig2015ConvergenceAO}, problem \cref{problem: variational obstacle discretized} is solved for a constant coefficient $\kappa \equiv 1$. The domain is chosen as $D = [0,1]^2$ and the constant forcing is set to $f\equiv 25$. The obstacle, which is used to model rough surfaces~\cite{Persson_2005}, is given by 
    \begin{align*}
        \obs(x,\bfy) = \sum_{q} B_q(H) \cos(q\cdot x + \bfy_q),
    \end{align*}
    where the sum is taken over all $q$ with components, which are multiples of $\pi$ such that $1\leq\norm{q}_2\leq 26$. The phase shifts $\bfy_q\sim \mathcal{U}([0,2\pi])$ and the Hurst exponent $H\sim\mathcal{U}([0,1])$ are mutually independent and the amplitudes are defined as $B_q(H) = \pi (2\pi\max(\norm{q}_2, 10))^{-(H+1)}/25$. The parameter vector $\bfy$ is set to the collection of $H$ and $\bfy_q$ for all considered $q$. A realization of the rough surface and corresponding solution can be found in \Cref{fig: rough obstacle}.
\end{enumerate}
\begin{figure}
    \centering
    \includegraphics[width=0.33\linewidth]{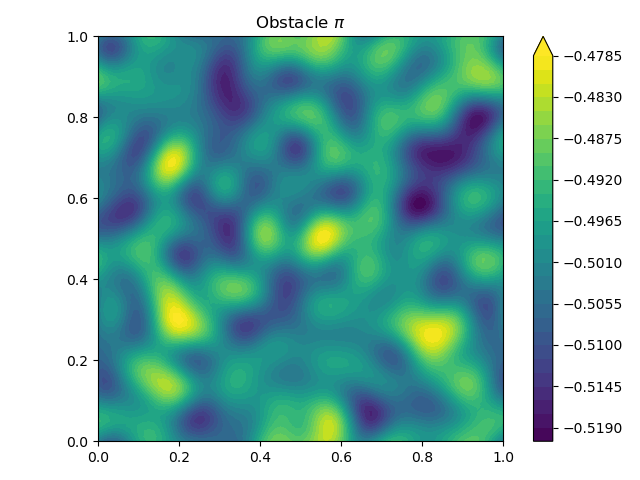}
    \includegraphics[width=0.34\linewidth]{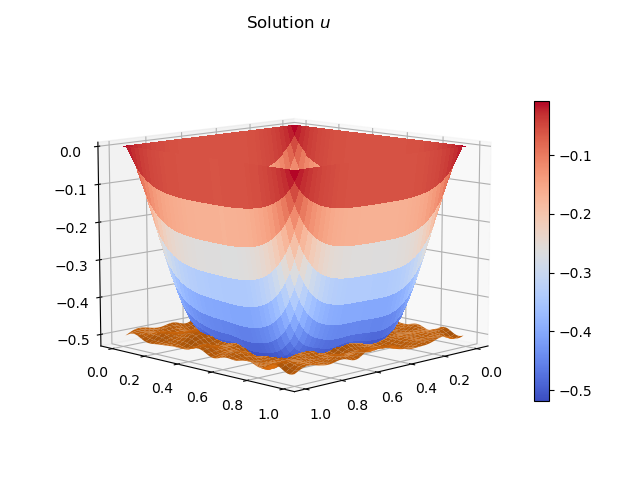}
    \includegraphics[width=0.32\linewidth]{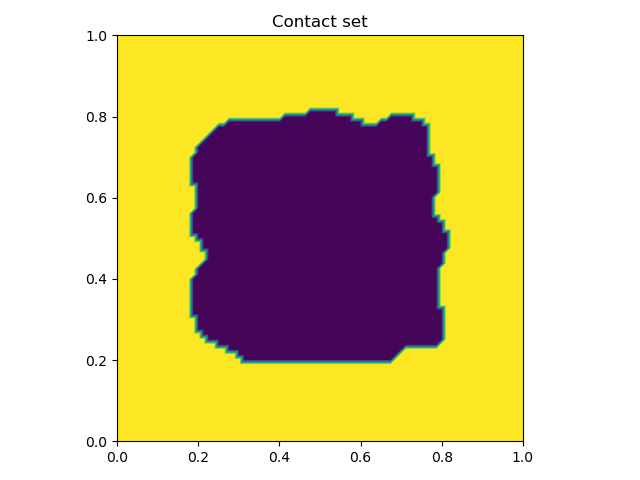}
    \caption{The first image depicts a realization of the rough surface model~\cite{Persson_2005}. In the second and third images, the corresponding solution of the obstacle problem and the resulting contact set are shown, where the solution is equal to the obstacle. The contact set is colored in purple.}
    \label{fig: rough obstacle}
\end{figure}

In each setting the problem was solved on $L=7$ levels with function spaces of size $ (5\times 2^{\ell-1} + 1) \times (5\times 2^{\ell-1} + 1)$ for $\ell=1,\dots,L$. The networks was trained with $10.000$ training samples and $1024$ validation samples. The number of parameters of the network on each level were selected as shown in~\cref{tab:NN_arch}. Note that the network architecture is chosen to approximate a coarse grid solution on the first level and finer grid corrections on higher levels. On finer grids, more FE coefficients need to be estimated than on coarser grids, but the contributions of the finer corrections to the full solutions are smaller than later corrective terms, see~\cref{section: multilevel}. Since the exponentially increasing number of FE coefficients to approximate on each level is countered by the exponential increase of required accuracy as derived in~\cref{equation:multilevel-error}, approximately the same number of parameters on each level was chosen for the CNN. Furthermore, note that on lower levels $\ell$, the $\vcmrm$ has fewer recursive calls than on higher levels. Since one U-net with parameters on every considered level corresponds to one full call of the $\vcmrm$, U-nets for lower levels need less parameters for a full downsampling and upsampling scheme. Therefore, on lower levels more U-nets are applied assuming the same number of trainable parameters as on higher levels.
\begin{table}[]
    \centering
    \caption{Collection of parameters of the used CNN architecture CNN, including the network output dimensions of the FE solutions of the obstacle problem on different levels. For the complete solution, all outputs and therefore all parameters on all levels are needed (and summed up). Moreover, the number of parameters used in the form of concatenated U-nets on each level is displayed.\newline}
    \begin{tabular}{cccccccc}
    \toprule
         level & $1$& $2$& $3$& $4$& $5$& $6$& $7$ \\
         \midrule
         \# params &  $1073248$ & $1069984$ & $1032992$ & $1014496$ & $838048$ & $996000$ & $1153952$\\
         \# U-nets & $11$ & $6$ & $4$ & $3$ & $2$ & $2$ & $2$\\
         \bottomrule
    \end{tabular}
    \label{tab:NN_arch}
\end{table}

For each test case, the \emph{mean relative $H^1(D)$} error and \emph{mean relative $L^2(D)$} error are calculated with respect to a finite element solution on the same grid as the neural network output and with respect to a (reference) fine grid solution. For $N=1024$ test samples $\bfy_1,\dots, \bfy_N \in \Gamma$, let $u_1,\dots,u_N$ be the output of the neural network in $V_L$ and let $v_L:\Gamma \to V_L$ be the finite element solution operator on the same grid. Furthermore, let $v_{\text{ref}}:\Gamma \to H^1$ be the finite element solution operator on a grid two times finer than the grid of $v_L$. Then, define the same grid network error and the reference error by 
\begin{align*}
    \mathcal{E}_{\text{MR}_*} = \sqrt{  \frac{\sum_{i=1}^N\norm{u_i - v_L(\bfy_i)}^2_*}{\sum_{i=1}^N\norm{v_L(\bfy_i)}^2_*}}, \qquad \mathcal{E}_{\text{MR}_*}^{\text{ref}} = \sqrt{  \frac{\sum_{i=1}^N\norm{u_i - v_{\text{ref}}(\bfy_i)}^2_*}{\sum_{i=1}^N\norm{v_{\text{ref}}(\bfy_i)}_*^2}}
\end{align*}
with $*\in\{H^1,L^2\}$.

\begin{table}
\centering
 \caption{The mean relative $H^1$ error is reported for the error of the network to Galerkin solutions on the same grid $\mathcal{E}_{\text{MR}{H^1}}$ and with respect to the reference Galerkin solutions on a twice uniformly refined grid $\mathcal{E}_{\text{MR}{H^1}}^{\text{ref}}$.\\}
 \begin{tabular}{c c c c} 
 \toprule
 problem & {parameter dimension $p$} & \quad$\mathcal{E}_{\text{MR}{H^1}}$ \quad & \quad$\mathcal{E}_{\text{MR}{H^1}}^{\text{ref}}$ \\
 \midrule
 deterministic obstacle & $10$ & $4.76 \times 10^{-4} \pm 2\times 10^{-4}$ & $6.82 \times 10^{-3} \pm 1\times 10^{-5}$\\
 & $50$ & $4.35 \times 10^{-4} \pm 1\times 10^{-4}$ & $6.82\times 10^{-3}\pm 8\times 10^{-6}$ \\
 \midrule
 stochastic obstacle & $11$ & $1.32\times 10^{-3} \pm 4\times 10^{-4}$ & $6.92\times 10^{-3} \pm 7\times 10^{-5}$\\
 & $51$ & $1.91 \times 10^{-3} \pm 1 \times 10^{-3}$ & $7.17 \times 10^{-3} \pm 4 \times 10^{-4}$\\ 
 \midrule
 rough surface & $100$ & $2.13\times10^{-3}\pm6\times10^{-4}$& $8.59\times10^{-3}\pm6\times10^{-4}$\\
 & $220$ & $2.07\times10^{-3}\pm2\times10^{-4}$ & $9.04\times 10^{-3}\pm 3\times 10^{-5}$ \\
 \bottomrule
 \end{tabular}
 \label{table:h1}
\end{table}

\begin{table}
\centering
 \caption{The mean relative $L^2$ error is reported for the error of the network to Galerkin solutions on the same grid $\mathcal{E}_{\text{MR}{L^2}}$ and to reference Galerkin solutions on a twice uniformly refined grid $\mathcal{E}_{\text{MR}{L^2}}^{\text{ref}}$.\\}
 \begin{tabular}{cccc} 
 \toprule
 problem & {parameter dimension $p$} & \quad $\mathcal{E}_{\text{MR}{L^2}}$ \quad & \quad $\mathcal{E}_{\text{MR}{L^2}}^{\text{ref}}$\\
 \midrule
 deterministic obstacle & $10$ & $1.85 \times 10^{-4} \pm 5\times 10^{-5}$ & $1.96 \times 10^{-4} \pm 5\times 10^{-5}$\\
 & 50 & $1.41 \times 10^{-4} \pm 4\times 10^{-5}$ & $1.52 \times 10^{-4} \pm 4\times 10^{-5}$ \\
 \midrule
 stochastic obstacle & $11$ & $6.39 \times 10^{-4} \pm 2 \times 10^{-4}$& $6.43 \times 10^{-4} \pm 2 \times 10^{-4}$\\ 
 & $51$ & $1.01\times 10^{-3} \pm 7 \times 10^{-4}$ & $1.01\times 10^{-3} \pm 6 \times 10^{-4}$ \\
 \midrule
 rough surface & $100$ & $7.00\times10^{-4}\pm4\times10^{-4}$ & $7.09\times10^{-4}\pm4\times10^{-4}$ \\
 & $220$ & $5.40 \times 10^{-4}\pm 3\times10^{-5}$ & $5.47 \pm 10^{-4}\pm 3\times10^{-5}$ \\
 \bottomrule
 \end{tabular}
 \label{table:l2}
\end{table}

The training of the CNN was repeated $5$ times and the mean results as well as the variances over the training procedures are recorded in \cref{table:h1} and \cref{table:l2}. \cref{table:h1} shows the the errors in the $H^1$ norm. It can be observed that in all test cases the network error is significantly smaller than the reference error, in some cases it is even a magnitude smaller. Therefore, the overall (prediction or approximation) error could only efficiently be reduced by refining the grid.
Furthermore, note that the error does not increase with respect to the parameter dimension $p$ despite the problem becoming more challenging.

For the $L^2$ error in \cref{table:l2}, the domination of the FE approximation error is not as pronounced. This could be due to the fact that the network is trained with respect to the $H^1$ error. The difference in error convergence can also be seen in \Cref{fig: decay over dofs}, where the decay of the $H^1$ and $L^2$ errors for the obstacle problem with variable obstacle and parameter dimension $11$ over the degrees of freedom is visualized.
Both the error of the network and the reference solutions ($\norm{u_i - v_{\text{ref}}(\bfy_i)}_*$) and the error of the FE solution on the same grids as the network output and the reference solution ($\norm{v_L(\bfy_i) - v_{\text{ref}}(\bfy_i)}_*$) are considered. While the $H^1$ error achieves the same convergence rate as the true Galerkin solution on each refined grid, the $L^2$ error suffers from a bias in the last few levels. 
\begin{figure}
    \centering
    \begin{subfigure}{0.4\textwidth}
        \includegraphics[width=\textwidth]{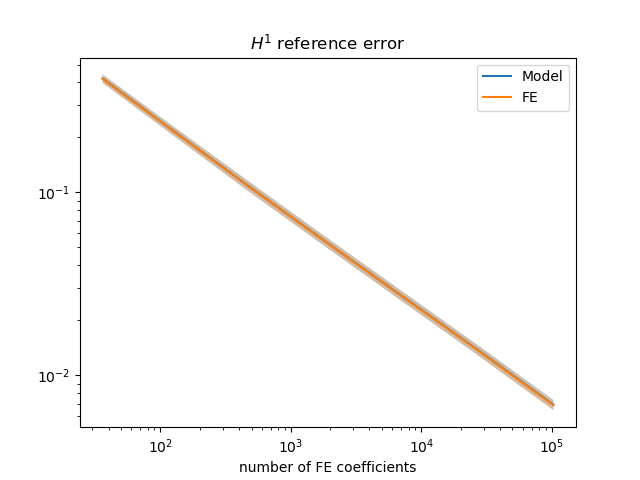}
    \end{subfigure}
    \begin{subfigure}{0.4\textwidth}
        \includegraphics[width=\textwidth]{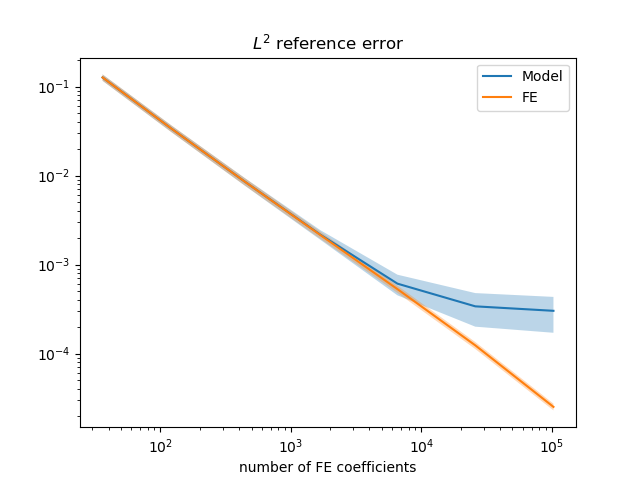}
    \end{subfigure}
    \caption{Error plots for the stochastic constant obstacle problem with parameter dimension $p=11$ are shown for a trained CNN. Errors of the CNN output compared to a reference solution are plotted in blue and errors of the finite element solution on the same grid as the CNN output to the reference solution are plotted in orange. A line indicates the mean of the relative errors over a test set and the area visualizes its variance. The left plot shows $H^1$ errors and the right plot shows $L^2$ errors.}
    \label{fig: decay over dofs}
\end{figure}

The error contribution of the individual levels is illustrated in \Cref{fig: partial errors}. The mean relative $H^1$ and $L^2$ errors are visualized for the approximation of the correction in each discrete FE space. The approximated corrections are visualized in~\Cref{fig: multilevel}. It can be seen that a high relative accuracy is achieved for the coefficients of corrections on high fidelity grids despite approximately the same number of parameters  being used for each sub-network. This underlines the effectiveness of the multigrid decomposition.
\begin{figure}
    \centering
    \begin{subfigure}{0.4\textwidth}
        \includegraphics[width=\linewidth]{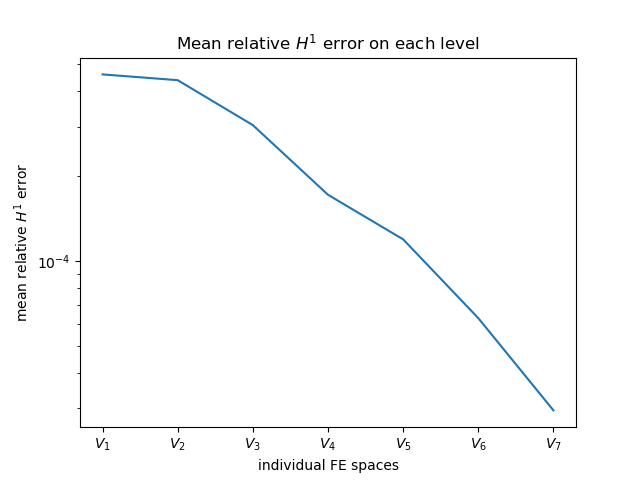}
    \end{subfigure}
    \begin{subfigure}{0.4\textwidth}
        \includegraphics[width=\linewidth]{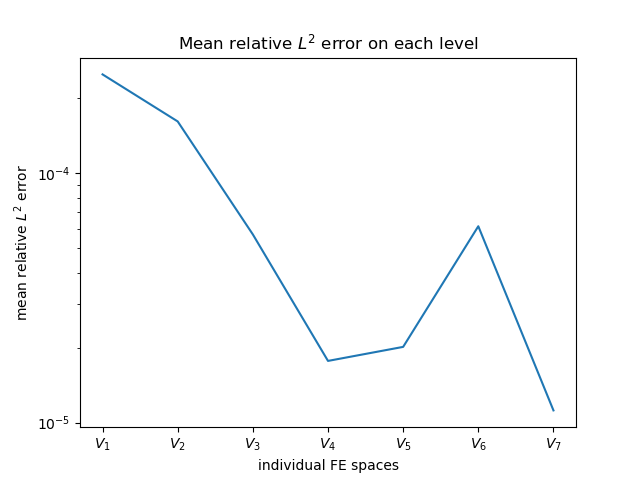}
    \end{subfigure}
    \caption{Mean relative $H^1$ (left) and $L^2$ (right) errors for a CNN trained for the stochastic constant obstacle problem with parameter dimension $p=11$ are shown. The errors are plotted over the outputs of the network on each level in the multi-level decomposition. It can be seen that the multigrid corrections on fine grids are well approximated.}
    \label{fig: partial errors}
\end{figure}

\section{Conclusion}
A multi-level NN architecture was analyzed in the context of a parametric obstacle problem, a highly nonlinear and highly challenging extension to the studies in \cite{JMLR:v24:23-0421}.
Here, it could be shown analytically and numerically that their proposed architecture is well suited to solve the considered parametric obstacle problem.
An expressivity result for the architecture applied to the obstacle problem was derived, stating an upper bound on the number of trainable parameters, which depends only logarithmically on the required accuracy. 
The architecture was successfully applied to the parametric obstacle problem for multiple parameter-setups with different dimensions, with variable coefficient and with constant as well as variable and rough obstacles.

\section*{Acknowledgments}
ME \& JS acknowledge funding from the Deutsche Forschungsgemeinschaft (DFG, German Research Foundation) in
the priority programme SPP 2298 ”Theoretical Foundations of Deep Learning”. ME acknowledges support by the ANR-DFG project COFNET: Compositional functions networks - adaptive learning for high-dimensional
approximation and uncertainty quantification. 
This study does not have any conflicts to disclose.

\bibliographystyle{abbrv}
\bibliography{lib}

\appendix
\newpage
\section{Proofs of expressivity theory}\label{section: proof of CNN for obstacle}

\begin{proof}[Proof of \cref{theorem:omega>gamma}]
    First, assume some $\lambda: D \to \mathbb{R}$ with $\lambda = 0$ on $\partial D$ that is continuous and Lebesgue-almost-everywhere differentiable.
    If $\lambda$ is not constant zero this implies that there exists a point $x_0\in D$ and $\varepsilon>0$ such that $\lambda$ is differentiable and $\nabla \lambda \neq 0$ on an $\varepsilon$ neighborhood of $x_0$ denoted b $U_\varepsilon(x_0)$.
    Then, due to $\kappa(\cdot, \bfy) >0$ everywhere, we obtain that
    \begin{align*}
        a_{\bfy} (\lambda, \lambda) = \int_{D} \kappa(\cdot, \bfy) \scpr{\nabla \lambda, \nabla \lambda} \mathrm{d}x \geq \int_{U_\varepsilon(x_0)} \kappa(\cdot, \bfy) \scpr{ \nabla \lambda, \nabla \lambda} \mathrm{d}x >0.
    \end{align*}
    Therefore, for any nonzero $\bfw \in\mathbb{R}^{N}$ 
    we deduce that
    \begin{align*}
        \bfw^\intercal A_\bfy \bfw = a_{\bfy, k}\left(\sum_{i=1}^{N}\bfw_i\lambda_i, \sum_{i=1}^{N}\bfw_i\lambda_i\right) >0
    \end{align*}
    and hence $A_\bfy$ is positive definite.
    Denote the eigenvalues and eigenvectors of $A_\bfy$ by $\sigma_i, \bfv^i$ for $i=1,\dots,N$ with
    \begin{align*}
        A_\bfy \bfv^i &= \sigma_i \bfv^i \quad\quad\text{ such that }\\
        \delta_{i,j} &= \scpr{\bfv^i,\bfv^j}_{\ell^2} \text{ for all } i,j=1,\dots,N.
    \end{align*}
    Furthermore, for $\bfw\in \mathbb{R}^{N}$, let 
    \begin{align*}
        J_\bfy\bfw &\coloneqq (I - \omega_\bfy A_\bfy)\bfw.
    \end{align*}
    Then, with $\bfw = \sum_{i=1}^{N} c_i\bfv^i$ 
    \begin{align*}
        J_\bfy\bfw = \sum_{i=1}^{N} c_i (I - \omega_\bfy A_\bfy)\bfv^i = \sum_{i=1}^{N} c_i (1 - \omega_\bfy \sigma_i)\bfv^i.
    \end{align*}
    Additionally,
    \begin{align*}
        |\bfw|^2 \coloneqq \sum_{i=1}^{N} \sigma_i (1-\sigma_i\omega_\bfy) c_i^2
    \end{align*}
    defines a semi-norm due to $0<\omega_\bfy \leq \sigma_{\max}(A_\bfy)^{-1}$.
    It then follows that
        \begin{align*}
            |\bfw|^2 &= \sum_{i,j=1}^{N} (1-\sigma_i\omega_\bfy) \sigma_i\scpr{\bfv^i,\bfv^j}_{\ell^2} c_i c_j 
            = \scpr{ \sum_{i=1}^{N} c_i (1-\sigma_i\omega_\bfy)\bfv^i, \sum_{j=1}^{N} c_j\bfv^j }_{A_\bfy}
            = \scpr{J_\bfy\bfw, \bfw}_{A_\bfy},\\
            \norm{\bfw}_{A_\bfy}^2 &= \scpr{A_\bfy\bfw,\bfw}_{\ell^2} = \sum_{i,j=1}^{N} \sigma_i c_i c_j \scpr{\bfv^i,\bfv^j}_{\ell^2} = \sum_{i=1}^{N} c_i^2 \sigma_i,\\
            |\bfw|^2 &= \sum_{i=1}^{N} \sigma_i c_i^2 - \sum_{i=1}^{N} \sigma_i\omega_\bfy c_i^2 { = (1-\omega_\bfy)}\sum_{i=1}^{N} \sigma_i c_i^2 = { (1-\omega_\bfy)}\norm{\bfw}_{A_\bfy}^2.
        \end{align*}
        With the H\"older inequality,
        \begin{align*}
            \norm{J_\bfy\bfw}_{A_\bfy}^2 &= \sum_{i=1}^{N}\sigma_i (c_i(1-\sigma_i\omega_\bfy))^2 = \sum_{i=1}^{N} (\sigma_i^{1/3} |c_i|^{2/3}) (\sigma_i^{2/3} |c_i|^{4/3}(1-\sigma_i\omega_\bfy)^2)\\
            &\leq \left(\sum_{i=1}^{N} (\sigma_i^{1/3} |c_i|^{2/3})^3\right)^{1/3} \left(\sum_{i=1}^{N}(\sigma_i^{2/3} |c_i|^{4/3}(1-\sigma_i\omega_\bfy)^2)^{3/2}\right)^{2/3}\\
            &= \left(\sum_{i=1}^{N} \sigma_i |c_i|^{2}\right)^{1/3} \left(\sum_{i=1}^{N}\sigma_i |c_i|^{2}(1-\sigma_i\omega_\bfy)^3\right)^{2/3}.
        \end{align*}
        This yields the result by estimating
        \begin{align*}
            \norm{J_\bfy\bfw}_{A_\bfy}^3 &= \left(\norm{J_\bfy\bfw}_{A_\bfy}^2\right)^{3/2}\\
            &\leq \left(\sum_{i=1}^{N} \sigma_i c_i^{2}\right)^{1/2} \left(\sum_{i=1}^{N}\sigma_i c_i^{2}(1-\sigma_i\omega_\bfy)^3\right)\\
            &= \norm{\bfw}_{A_\bfy} |J_\bfy\bfw|^2\\
            &= \norm{\bfw}_{A_\bfy} { (1-\omega_\bfy)}\norm{J_\bfy\bfw}_{A_\bfy}^2
        \end{align*}
        and dividing by $\norm{J_\bfy\bfw}_{A_\bfy}^2$.
\end{proof}

\begin{lemma}[Maxima approximation]\label{lemma: CNN max approximation}
Let $\tau:\mathbb{R}\to \mathbb{R}$ satisfy~\cref{ass: sigma} $\varepsilon, M >0$ and $n\in\mathbb{N}$. There exist convolutional $(1,1)$-kernels $K_1\in\mathbb{R}^{2\times 2\times 1\times 1}$ and $K_2\in\mathbb{R}^{2\times 1\times 1\times 1}$ and biases $B_1\in\mathbb{R}^2, B_2\in\mathbb{R}$ such that 
\begin{align*}
    \norm{(\psi_{K_2,B_2} \circ \tau \circ \psi_{K_1,B_1})(\bfu,\bfobs) - \max\{\bfu,\bfobs\}}_{L^\infty([-M,M]^{2 \times n \times n})} \leq\varepsilon,
\end{align*}
where the maximum is defined component-wise.    
\end{lemma}

\begin{proof}
    Note that $\max\{\bfu,\bfobs\} = \max\{\bfu - \bfobs, 0\} + \bfobs$ and $\bfobs = \max\{\bfobs,0\}+\max\{-\bfobs,0\}$. Therefore, an approximation of the mapping
    \begin{align*}
        \begin{pmatrix}
            \bfu\\
            \bfobs
        \end{pmatrix}\mapsto 
        \begin{pmatrix}
            \bfu-\bfobs\\
            \bfobs
        \end{pmatrix}\mapsto
        \begin{pmatrix}
            \max\{\bfu - \bfobs, 0\}\\
            \max\{\bfobs,0\}\\
            \max\{-\bfobs,0\}
        \end{pmatrix}\mapsto
        \begin{pmatrix}
            \max\{\bfu - \bfobs, 0\} + \bfobs
        \end{pmatrix}
    \end{align*}
    yields the claim. 
    The addition of channels in the first and last step of the flow can be represented by kernels of width and height $1$ with the appropriate number of input and output channels.
    According to~\cref{ass: sigma}, there exist affine linear mappings $\rho_1,\rho_2:\mathbb{R}\to\mathbb{R}$ such that the maximum of the input and $0$ can be approximated on $[-2M,2M]$.
    Accounting for $\bfu-\bfobs\in[-2M,2M]$ for $\bfu,\bfobs\in[-M,M]$ implies that applying these maps to the appropriate channels provides an approximation of the second step. Concatenating these maps shows the claim since the kernels have width and height $1$ and can therefore be concatenated to one kernel of width and height $1$.
\end{proof}

\begin{proof}[Proof of~\cref{theorem: CNN for richardson}]
    Let $\T_h$ be a uniform triangulation of $D$, where each node $i$ is in $6$ triangles $T_i^1,\dots, T_i^6$ as depicted in \cite[Figure 4]{JMLR:v24:23-0421} and let $V_h$ be the P1 finite element space over the triangulation. As in \cite[Definition 14]{JMLR:v24:23-0421}, let $\bfups(\kappa_h,k,i)\coloneqq \int_{T_i^k} \kappa_h \dx{x}$ and $\bfups(\kappa_h) \coloneqq [\bfups(\kappa_h,k,i)]_{k\in[6],i\in[\dim V_h]}\in\mathbb{R}^{6\times \dim V_h}$.
    Moreover, let $F:\mathbb{R}^{7\times \dim V_h} \to \mathbb{R}^{\dim V_h}$ be defined as in~\cite[Theorem 16]{JMLR:v24:23-0421}.
    with $F(\bfu,\bfups(\kappa_h)) = A_\kappa \bfu$. 
    We consider a CNN approximating the following steps.
    \begin{align*}
        \begin{pmatrix}
            \bfu^{(0)}\\
            \bfkappa\\
            \bff\\
            \bfobs
        \end{pmatrix}\hspace{-.5ex}\mapsto\hspace{-.5ex}
        \begin{pmatrix}
            \bfu^{(0)}\\
            \bfups(\kappa)\\
            \bff\\
            \bfobs
        \end{pmatrix}\hspace{-.5ex}\mapsto\hspace{-.5ex}
        \begin{pmatrix}
            \bfu^{(1)} \coloneqq \max\left\{\bfu^{(0)} + \omega\left[ \bff - F( \bfu^{(0)}, 
            \bfups(\kappa))\right], \bfobs\right\}\\
            \bfups(\kappa)\\
            \bff\\
            \bfobs
        \end{pmatrix}
        \hspace{-.5ex}\mapsto\hspace{-.5ex} \dots\hspace{-.5ex} 
        \mapsto\hspace{-.5ex}
        \begin{pmatrix}
            \bfu^{(m)}\\
            \bfups(\kappa)\\
            \bff\\
            \bfobs
        \end{pmatrix}\hspace{-.5ex}\mapsto\hspace{-.5ex}
        \begin{pmatrix}
            \bfu^{(m)}
        \end{pmatrix}
    \end{align*}
    According to \cite[Lemma 15(i)]{JMLR:v24:23-0421}, there exists a convolutional kernel $K_1\in\mathbb{R}^{1\times 6\times 3\times 3}$ such that $\bfkappa \ast K_1 = \bfups(\kappa_h)$. Trivially extending the kernel to unused channels and parallelizing with identity kernels yields a CNN representation of the first step.
    According to \cite[Theorem 18]{JMLR:v24:23-0421}, for any $\tilde\varepsilon,\tilde M>0$ there exists a CNN $\Psi_{\tilde \varepsilon, \tilde M}:\mathbb{R}^{7\times \dim V_h} \to \mathbb{R}^{\dim V_h}$ of size independent of $\tilde\varepsilon, \tilde M$ such that
    \begin{align*}
        \norm{ \Psi_{\tilde \varepsilon, \tilde M} - F}_{L^\infty ([-\tilde M,\tilde M]^{7\times \dim V_h})} \leq \tilde\varepsilon.
    \end{align*}
    Furthermore,~\cref{lemma: CNN max approximation} shows, that the maximum can be approximated by a CNN.
    Again extending the CNN to unaltered channels and parallelizing with an identity CNN yields approximations of every intermediate step.
    The last step is realized by the convolution with a kernel $K_2\in\mathbb{R}^{9\times 1\times 1\times 1}$, which is $1$ in the first and $0$ in all other input channels.
    
    According to \cite[Lemma 20]{JMLR:v24:23-0421} and since all steps are continuous operations, their concatenation can be approximated by the concatenation $\tilde\Psi$ of the approximating CNNs $\Psi_{\tilde\varepsilon,\tilde M}$, where different approximation accuracies $\tilde \varepsilon$ and domains $\tilde M$ are expected for the different steps. Concatenating a one layer CNN with kernel $K_1$ and $\tilde\Psi$ and another one layer CNN with kernel $K_2$ yields a CNN $\Psi$.
    Then, for each $\bfu^{(0)},\bfkappa,\bff,\bfobs\in [-M,M]^{n\times n}$ it holds that
    \begin{align*}
        \norm{\Psi(\bfu^{(0)}, \bfkappa, \bff, \bfobs) - \bfu^{(m)}}_{\infty} \leq \varepsilon.
    \end{align*}
    Since the space of FE coefficients is a finite dimensional real vector space, the $\norm{\cdot}_\infty$-norm is equivalent to the $\norm{\cdot}_{H^1}$-norm.
    Since the size of the approximating CNNs does not depend on $\tilde\varepsilon,\tilde M$, the size of the concatenated CNN also does not depend on it and the overall number of parameters only grows linearly with the number of intermediate steps $m$.
\end{proof}

\begin{proof}[Proof of~\cref{theorem: full CNN from richardson}]
    Let $\omega \coloneqq (\ub \chl)^{-1}$. Then \cref{lemma:gamma} yields $\norm{I - \omega A_\bfy}_{A_\bfy} <1-\omega$ for all $\bfy\in\Gamma$.
    For $\bfy\in\Gamma$, \eqref{equation: Richardson decay} implies for $A \coloneqq A_\bfy$ and for $\bfe^{(m)} \coloneqq \bfu^{(m)} - \bfu$ (where $\bfu$ solves \cref{problem: parametric obstacle discretized}) 
    that for $\gamma \coloneqq 1-\omega$,
    \begin{align}\label{eq: bound richardson}
        \norm{\bfe^{(m)}}_{A_\bfy} & \leq \gamma^m \norm{\bfobs - \bfu}_{A_\bfy}.
    \end{align}
    Note that the second term can be bounded by the following consideration.
    \begin{align*}
        \norm{\bfobs - \bfu}_{A_\bfy}^2&= \int_D\kappa_h(\cdot,\bfy) \scpr{\nabla (\obs_h - u_h), \nabla(\obs_h - u_h)}\dx{x}\\
        &= -\int_D\kappa_h(\cdot,\bfy) \scpr{\nabla u_h, \nabla(\obs_h - u_h)}\dx{x} + \int_D\kappa_h(\cdot,\bfy) \scpr{\nabla \obs_h, \nabla(\obs_h - u_h)}\dx{x}.
    \end{align*}
    Since $u_h$ solves \cref{problem: variational obstacle discretized}, this implies 
    \begin{align*}
        \norm{\obs_h - u_h}^2_{A_\bfy} &= \int_D \kappa(\cdot,\bfy)\scpr{\nabla (\obs_h - u_h), \nabla(\obs_h - u_h)} \dx{x}\\
        &= -\int_D \kappa(\cdot,\bfy)\scpr{\nabla u_h, \nabla(\obs_h - u_h)} \dx{x} + \int_D \kappa(\cdot,\bfy)\scpr{\nabla \obs_h, \nabla(\obs_h - u_h)} \dx{x}\\
        &\leq -\int_D f(x)(\obs_h - u_h)(x)\dx{x} +\norm{ \obs_h}_{A_\bfy} \norm{\obs_h - u_h}_{A_\bfy}\\
        &\leq  \norm{f}_{H^{-1}} \norm{ \obs_h - u_h}_{H^1} + \frac{1}{c_{H^1}} \norm{\obs_h}_{H^1} \norm{\obs_h - u_h}_{A_\bfy}\\
        &\leq C_{H^1} \norm{f}_{H^{-1}} \norm{ \obs_h - u_h}_{A_\bfy} + \frac{1}{c_{H^1}} \norm{\obs_h}_{H^1} \norm{\obs_h - u_h}_{A_\bfy},
    \end{align*}
    where the fist inequality follows with the definition of the dual norm~\eqref{eq: dual norm} and the Cauchy-Schwarz inequality. The second and last inequality follow from~\eqref{eq:norm equivalence A and H1}.
    With~\eqref{eq: A_y norm equal} the second term is bounded by 
    \begin{align}\label{eq: bound on obs-u}
        \norm{ \bfobs - \bfu}_{A_\bfy} = \norm{\obs_h - u_h}_{A_\bfy} \leq C_{H^1} \norm{f}_{H^{-1}} + \frac{1}{c_{H^1}} \norm{\obs_h}_{H^1}.
    \end{align}
    To bound $\gamma^m$, $m$ can be chosen as
    \begin{align*}
        m \geq \log \left(\nicefrac{\varepsilon}{2}\right)\log\left(\gamma\right)^{-1},
    \end{align*}
    such that~\eqref{eq: bound richardson} and~\eqref{eq: bound on obs-u} lead to 
    \begin{align*}
        \norm{\bfu^{(m)} - \bfu}_{H^1} \leq \frac{1}{c_{H^1}} \norm{\bfu^{(m)} - \bfu}_{A_\bfy} \leq \frac{\varepsilon}{2 c_{H^1}} \left( C_{H^1} \norm{f}_{H^{-1}} + \frac{1}{c_{H^1}} \norm{\obs_h}_{H^1} \right).
    \end{align*}
    According to \cref{theorem: CNN for richardson} there exists a CNN $\Psi$ such that the number of parameters grows linearly with $m$ and
    \begin{align*}
        \norm{\Psi(\bfobs, \bfkappa, \bff, \bfobs) - \bfu^{(m)}}_{H^1} \leq \frac{\varepsilon}{2 c_{H^1}}\left( C_{H^1} \norm{f}_{H^{-1}} + \frac{1}{c_{H^1}} \norm{\obs_h}_{H^1} \right).
    \end{align*}
    This yields the claim with
    \begin{align*}
        \norm{\Psi(\bfobs, \bfkappa,\bff,\bfobs) - \bfu}_{H^1} &\leq  \norm{\Psi(\bfobs, \bfkappa, \bff, \bfobs) - \bfu^{(m)}}_{H^1} + \norm{\bfu^{(m)} - \bfu}_{H^1} \\
        &\leq \frac{\varepsilon}{c_{H^1}}\left( C_{H^1} \norm{f}_{H^{-1}} + \frac{1}{c_{H^1}} \norm{\obs_h}_{H^1} \right).
    \end{align*}
\end{proof}

\begin{theorem}[CNN for monotone restriction operator]\label{theorem: CNN for monotone restriction}
    For $\ell=1,\dots,L-1$, let $V_\ell\subset V_\ell+1\subset H_0^1([0,1]^2)$ be two nested P1 finite element spaces as in \cref{section: fem}. Let the activation function $\tau$ satisfy \cref{ass: sigma} and $\Rl$ be the monotone restriction operator defined as in \cref{definition: Rl}.
    There exists a $C>0$ such that for every $\varepsilon>0$ there exists a CNN $\Psi$ and kernel $K\in\mathbb{R}^{1\times 9\times 3\times 3}$ such that
    \begin{enumerate}[label=(\roman*)]
        \item $\norm{(\Psi\circ \ast^{2\text{s}}_K)(\bfobs) - \Rl \bfobs}_{L^\infty([-M,M]^{n\times n})} \leq\varepsilon$,
        \item the number of parameters is bounded by $M(\Psi) + M(K) \leq C$.
    \end{enumerate}
\end{theorem}
\begin{proof}
    Let $K\in\mathbb{R}^{1\times9\times3\times3}$ be defined such that for each output channel $i=1,\dots,9$ one kernel entry is one and all other are zero. Applying this kernel in a two-strided convolution, a FE function on a fine grid as the input image gets mapped to nine FE functions on the coarse grid, i.e. an output tensor with nine channels. For each grid point in the coarse grid, these nine channels contain the values of the grid points in the fine grid surrounding the same grid point. By the definition of the monotone restriction, the maximum needs to be taken over all channels in each pixel. This can be achieved by applying \cref{lemma: CNN max approximation}.
    Denote the maximum of $\bfobs_1,\bfobs_2$ by $\bfobs_{1,2}$. Let $\Psi_1,\dots,\Psi_4$ be CNNs approximating the steps
    \begin{align*}
        \begin{pmatrix}
            \bfobs_1\\
            \dots\\
            \bfobs_9
        \end{pmatrix} \mapsto
        \begin{pmatrix}
            \max\{\bfobs_1,\bfobs_2\} \\
            \max\{\bfobs_3,\bfobs_4\}\\
            \dots
        \end{pmatrix}\eqqcolon
        \begin{pmatrix}
            \bfobs_{1,2}\\
            \dots\\
            \bfobs_{7,8}\\
            \bfobs_9
        \end{pmatrix}\mapsto
        \begin{pmatrix}
            \bfobs_{1,2,3,4}\\
            \bfobs_{5,6,7,8}\\
            \bfobs_9
        \end{pmatrix}\mapsto
        \begin{pmatrix}
            \bfobs_{1,2,3,4}\\
            \bfobs_{5,6,7,8,9}
        \end{pmatrix}\mapsto
        \begin{pmatrix}
            \bfobs_{1,2,3,4,5,6,7,8,9}
        \end{pmatrix}.
    \end{align*}
    Since the concatenation of functions can be approximated by the concatenation of approximate functions~\cite[Lemma 20]{JMLR:v24:23-0421} and the size of the kernels does not depend on the required accuracy, the size of the concatenated CNN does not depend on the $\varepsilon$.
    Since each step can be approximated with two kernels of height and width one and one application of the activation function, the secondly applied kernel of $\Psi_k$ and the firstly applied kernel of $\Psi_{k+1}$ can be concatenated to one kernel performing both convolutions. This leads to a CNN $\Psi$ with $4$ applications of the activation functions and $5$ kernels of height and width one. This yields the claim.
\end{proof}

\begin{proof}[Proof of \cref{theorem: CNN for V-Cycle}]
The proof can be carried out similarly to~\cite[Proof of Theorem 6]{JMLR:v24:23-0421}. Here, the obstacle has to be taken into account in every step. For each $\ell=1,\dots,L$ and $g:\mathbb{R}^{n\times \dim V_\ell} \to \mathbb{R}^{m\times\dim V_\ell}$, let the function also considering the obstacle $\obs$ be defined by $\tilde{g}:\mathbb{R}^{(n+1)\times \dim V_\ell} \to \mathbb{R}^{(m+1)\times\dim V_\ell}$ with $(\bfw,\bfobs) \mapsto (g(\bfw),\bfobs)$. Adapting the proof to use the extended functions for $f_\text{in}, f_\text{sm}^\ell, f_\text{resi}^\ell$, only a few changes in the smoothing iterations and projections have to be made in the proof~\cite[Proof of theorem 6]{JMLR:v24:23-0421} as detailed in the following.
\begin{itemize}
    \item[(ii)] \textbf{Smoothing iteration:} For $\ell=1,\dots,L$ a function $f^\ell_{\text{sm}}$ is defined as a mapping from the current discrete solution $\bfu$, the integrated coefficient $\bfups$ and the discretized right-hand side $\bff$ to an updated discrete solution. The other inputs are unaltered ${f}_{\text{sm}}^\ell:\mathbb{R}^{8\times \dim V_\ell} \to \mathbb{R}^{8\times \dim V_\ell}$,
    \begin{align*}
        f_{\text{sm}}^\ell(\bfu, \bfups, \bff) = [\bfu+\omega ( \bff - A_\bfy^\ell \bfu ), \bfups, \bff].
    \end{align*}
    In~\cite[Theorem 18]{JMLR:v24:23-0421} it is shown that this function can be approximated up to an $L^\infty$ error bounded by some $\varepsilon>0$ on a compact subset of $D$, where the number of parameters is independent of the the size of the compact set and the accuracy $\varepsilon$. To use the projected Richardson method for the smoothing iterations, this function has to be adjusted to map another input to itself. For this, define $\tilde{f}_{\text{sm}}^\ell:\mathbb{R}^{9\times \dim V_\ell} \to \mathbb{R}^{9\times \dim V_\ell}$ with 
    \begin{align*}
        \tilde{f}_{\text{sm}}^\ell(\bfu, \bfups, \bff, \bfobs) = [\bfu+\omega ( \bff - A_\bfy^\ell \bfu ), \bfups, \bff, \bfobs].
    \end{align*}
    Next, the maximum of the first and last input has to be taken. Hence, let $f_{\max}^\ell:\mathbb{R}^{9\times \dim V_\ell} \to \mathbb{R}^{9\times \dim V_\ell}$ be defined by
    \begin{align*}
        f_{\max}^\ell(\bfu, \bfups, \bff, \bfobs) = [\max \{\bfu,\bfobs\}, \bfups, \bff, \bfobs].
    \end{align*}
    Then, \cref{lemma: CNN max approximation} implies that $f_{\max}^\ell$ can be approximated up to accuracy $\varepsilon>0$ on a compact set in the $L^\infty$ norm with the number of parameters independent of the accuracy $\varepsilon$ and the size of the compact set. Concatenating the last two functions yields an update of the projected Richardson iteration \eqref{equation: projected richardson} of the form
    \begin{align*}
        (f_{\max}^\ell\circ \tilde{f}_{\text{sm}}^\ell)(\bfu, \bfups, \bff, \bfobs) = [\max\{\bfu+\omega ( \bff - A_\bfy^\ell \bfu ),\bfobs\}, \bfups, \bff, \bfobs]
    \end{align*}
    and can be approximated similarly to \cite[Lemma 20]{JMLR:v24:23-0421}.

    \item[(iii)] \textbf{Restricted residual:} The residual can be calculated as in~\cite[Proof of Theorem 6 (iii)]{JMLR:v24:23-0421} by altering the function to include the obstacle $\tilde{f}_{\text{resi}}^\ell$.
    The restriction function $f_{\text{rest}}$ is adjusted to also restrict the obstacle using the monotone restriction operator for the obstacle input. 
    For each $\bfu,\bfr,\bff, \bfobs\in\mathbb{R}^{\dim V_\ell}, \bar{\bfkappa}\in\mathbb{R}^{6\times\dim V_\ell}$, define
    \begin{align*}
        \tilde{f}_{\text{rest}}^\ell: \mathbb{R}^{10\times\dim V_\ell} \to \mathbb{R}^{8\times\dim V_{\ell}} \times \mathbb{R}^{8\times \dim V_{\ell-1}}, \begin{pmatrix}
            \bfu\\
            \bfr\\
            \bar{\bfkappa}\\
            \bff\\
            \bfobs
        \end{pmatrix} \mapsto \begin{bmatrix}
            \begin{pmatrix}
                \bfu\\
                \bar{\bfkappa}\\
                \bff\\
                \bfobs
            \end{pmatrix} \times
            \begin{pmatrix}
                \mathbf{0}\\
                \bar{\bfkappa} \ast K\\
                P^T_{\ell-1}\bfr\\
                \Rl \bfobs
            \end{pmatrix}
        \end{bmatrix},
    \end{align*}
    where $K = [K_1,\dots,K_6]\in\mathbb{R}^{6\times 6\times 3\times 3}$ is defined as in the proof mentioned above. There, it is shown that $\bfups(\kappa_h,\T^\ell,k)\ast K_k = \bfups(\kappa_h,\T^{\ell-1},k)$ for each $k=1,\dots, 6$ and except for $\Rl\bfobs$ each part can be represented by a convolutional kernel. The monotone restriction can be represented due to \cref{theorem: CNN for monotone restriction}.
\end{itemize}
As in \cite[Proof of Theorem 6]{JMLR:v24:23-0421}, the operations of one V-Cycle on level $\ell=2,\dots,L$ can then be written as the concatenation of the introduced functions
\begin{align*}
    \mathrm{VC}_{k_0,k}^\ell &\coloneqq \left(\bigcirc_{i=1}^k \left(f_{\max}^\ell\circ \tilde f_{\text{sm}}^\ell\right)\right)\circ\tilde f^\ell_{\text{prol}}\circ\left(\Id, \mathrm{VC}_{k_0,k}^{\ell-1}\right)\circ\tilde f_{\text{rest}}^\ell\circ\tilde f^\ell_{\text{resi}}\circ\left(\bigcirc_{i=1}^k (f_{\max}^\ell\circ \tilde f_{\text{sm}}^\ell)\right),\\
    \mathrm{VC}_{k_0,k}^1 &\coloneqq \bigcirc_{i=1}^{k_0} \tilde f_{\text{sm}}^1.
\end{align*}
To represent the whole algorithm, the input and output are adjusted like
\begin{align*}
    \vcmrm = \tilde f_\text{out} \circ \left(\bigcirc_{i=1}^m \mathrm{VC}_{k_0,k}^L\right)\circ \tilde f_{\text{in}}.
\end{align*}
\end{proof}

\end{document}